\documentclass{article}

\usepackage{microtype}
\usepackage{graphicx}
\usepackage{booktabs} 
\usepackage[numbers]{natbib}

\usepackage{hyperref}

\newcommand{\ours}{\textsc{DeSA}{}}
\newcommand{\ie}[0]{{\textit{i.e.},}}
\newcommand{\eg}[0]{{\textit{e.g.},}}

\usepackage[accepted]{icml2024}

\usepackage{amsmath}
\usepackage{amssymb}
\usepackage{mathtools}
\usepackage{amsthm}

\usepackage[noend]{algpseudocode}

\usepackage{multirow}
\usepackage{adjustbox}
\usepackage[flushleft]{threeparttable}

\usepackage{subfig}
\usepackage{wrapfig}

\newcommand{\x}{\mathbf{x}}

\newtheorem{thm}{Theorem}

\newtheorem{lem}{Lemma}
\newtheorem{claim}{Claim}
\newtheorem{prop}[thm]{Proposition}

\usepackage{pifont}
\newcommand{\cmark}{\ding{51}}%
\newcommand{\xmark}{\ding{55}}%

\newcommand{\ks}[1]{{\color{green}[Kartik : #1]}}

\newcommand{\xl}[1]{{\color{blue}[Xiaoxiao : #1]}}
\newcommand{\revision}[1]{{\color{blue}#1 }}

\definecolor{kb}{RGB}{248, 185, 60}
\newcommand{\CY}[1]{{\color{kb}[CY: #1]}}

\usepackage{amsmath,amsfonts,bm}

\def\eqref#1{equation~\ref{#1}}

\def\1{\bm{1}}

\DeclareMathAlphabet{\mathsfit}{\encodingdefault}{\sfdefault}{m}{sl}
\SetMathAlphabet{\mathsfit}{bold}{\encodingdefault}{\sfdefault}{bx}{n}

\DeclareMathOperator*{\argmin}{arg\,min}

\usepackage[capitalize,noabbrev]{cleveref}

\theoremstyle{plain}
\newtheorem{theorem}{Theorem}[section]

\theoremstyle{definition}
\newtheorem{definition}[theorem]{Definition}

\theoremstyle{remark}
\newtheorem{remark}[theorem]{Remark}

\usepackage[textsize=tiny]{todonotes}

\icmltitlerunning{Overcoming Data and Model heterogeneities in Decentralized Federated Learning via Synthetic Anchors}

\begin{document}

\twocolumn[
\icmltitle{Overcoming Data and Model heterogeneities in Decentralized Federated Learning via Synthetic Anchors}



\icmlsetsymbol{equal}{*}

\begin{icmlauthorlist}
\icmlauthor{Chun-Yin Huang}{eceubc,vector}
\icmlauthor{Kartik Srinivas}{eceubc,iit}
\icmlauthor{Xin Zhang}{meta}
\icmlauthor{Xiaoxiao Li}{eceubc,vector}
\end{icmlauthorlist}

\icmlaffiliation{eceubc}{Department of Electrical and Computer Engineering, The University of British Columbia, Canada}
\icmlaffiliation{vector}{Vector Institute, Canada}
\icmlaffiliation{meta}{Meta, U.S.A.}
\icmlaffiliation{iit}{Indian Institute of Technology Hyderabad, India (Work was done during internship at UBC)}

\icmlcorrespondingauthor{Xiaoxiao Li}{xiaoxiao.li@ece.ubc.ca}


\vskip 0.3in
]



\printAffiliationsAndNotice{}  

\begin{abstract}
Conventional Federated Learning (FL) involves collaborative training of a global model while maintaining user data privacy. One of its branches, decentralized FL, is a serverless network that allows clients to own and optimize different local models separately, which results in saving management and communication resources. Despite the promising advancements in decentralized FL, it may reduce model generalizability due to lacking a global model. In this scenario, managing data and model heterogeneity among clients becomes a crucial problem, which poses a unique challenge that must be overcome: \emph{How can every client's local model learn generalizable representation in a decentralized manner?}
To address this challenge, we propose a novel \textbf{De}centralized FL technique by introducing \textbf{S}ynthetic \textbf{A}nchors, dubbed as \textsc{DeSA}. Based on the theory of domain adaptation and Knowledge Distillation (KD), we theoretically and empirically show that synthesizing global anchors based on raw data distribution facilitates mutual knowledge transfer. We further design two effective regularization terms for local training: \emph{1) REG loss} that regularizes the distribution of the client's latent embedding with the anchors and \emph{2) KD loss} that enables clients to learn from others.
Through extensive experiments on diverse client data distributions, we showcase the effectiveness of \textsc{DeSA} in enhancing both inter- and intra-domain accuracy of each client. 
The implementation of \textsc{DeSA} can be found at: \url{https://github.com/ubc-tea/DESA}

\end{abstract}

\section{Introduction}

Federated learning (FL) has emerged as an important paradigm to perform machine learning from multi-source data in a distributed manner.
Conventional FL techniques leverage a large number of clients to process a global model learning, which is coordinated by a central server.
However, there arises concerns on increased vulnerability of system failures and trustworthiness concerns for the central server design in the conventional FL. An emerging paradigm, called \textbf{decentralized FL}, is featured by its serverless setting to address the issues.  Recent work has shown decentralized FL framework can provide more flexibility and solubility~\cite{beltran2023decentralized, yuan2023decentralized}, where they relax the use of central server for model aggregation. However, this could deflect the generalization capability of each client model. 
Although most of the works in decentralized FL focus on model personalization~\cite{huang2022learn}, we consider it crucial for decentralized FL to be generalizable since local training data may not align with local testing data in practice.

Client heterogeneity is a common phenomenon in FL that can deteriorate model generalizability.
On one hand, \textbf{data-heterogeneity} relaxes the assumption that the data across all the client are independent and identically distributed (i.i.d.). To solve the problem, a plethora of methods have been proposed. However, most of them assumes that the model architectures are invariant across clients \cite{li2019convergence,li2020federated,li2021model,karimireddy2020scaffold,tang2022virtual}. On the other hand, many practical FL applications (\eg{} Internet-of-Things and mobile device system) face \textbf{model-heterogeneity}, where clients have devices with different computation capabilities and memory constraints.
In conventional FL, strategies have been proposed to leverage knowledge transferring to address the model heterogeneity issue, \eg{} server collects labeled data with the similar distribution as the client data or clients transmit models~\cite{lin2020ensemble,zhu2021data}. However, these operations usually require a server to coordinate the knowledge distillation and assume global data is available~\cite{li2019fedmd, lin2020ensemble, tan2022fedproto}. Thus, they are not applicable to decentralized FL and may not fit into real-world scenarios. Recently, there are works proposing to perform test-time adaptation for out-of-federation clients~\cite{jiang2023iop}, while this paper focuses on the solution during FL training time. 

We can see that both heterogeneous FL and decentralized FL leave the gray space of the following practical research question:
\textit{How can \textbf{every} client model perform well on other client domains, in a completely decentralized and heterogeneous FL setup?} Such a problem is referred as \underline{\textit{decentralized federated mutual learning}}, which is further detailed in Section~\ref{sec:mutuallearning}. To the best of our knowledge, we are the first to address both data and model heterogeneity issues under serverless decentralized FL setting (see the comparison with related work in Table~\ref{tab:setting}).

\begin{figure}[t]
    \centering
    \includegraphics[width=\linewidth]{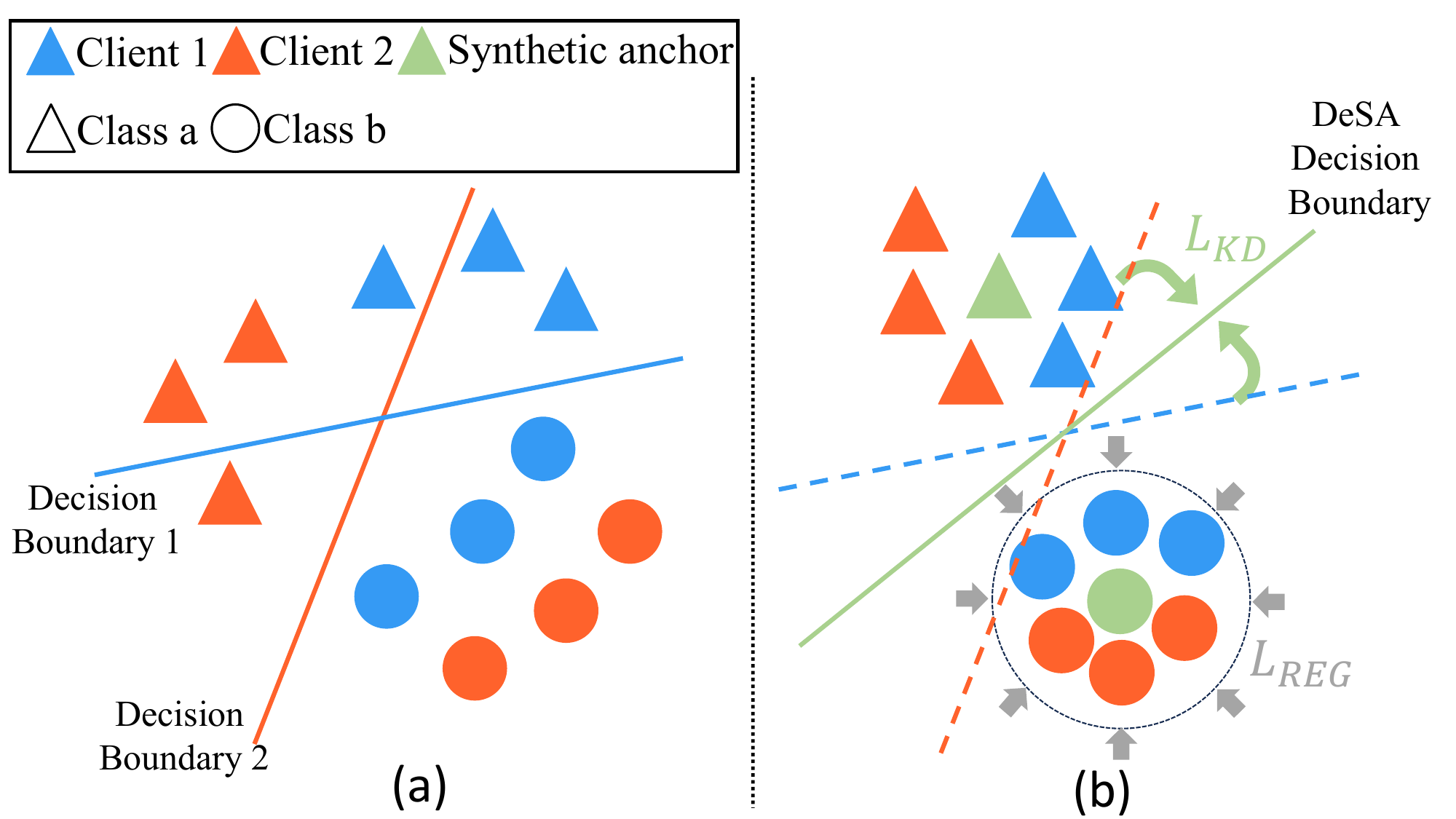}
    \caption{\small The decision boundary before (a) and after (b) applying our proposed $\mathcal{L}_{\rm REG}$ (Eq.~\ref{eq:lsgd}) and $\mathcal{L}_{\rm KD}$ (Eq.~\ref{eq:kd}) using our synthetic anchor data. $\mathcal{L}_{\rm REG}$ aims to group the raw feature towards synthetic anchor feature, and $\mathcal{L}_{\rm KD}$ twists the local decision boundary towards the generalized decision boundary.}
    \label{fig:loss_effect}
    \vspace{-5mm}
\end{figure}

In this work, we tackle the research question by performing local heterogeneity harmonized training and knowledge distillation. In particular, we synthesize a lightweight synthetic data generation process via distribution matching~\cite{zhao2023dataset}, and the synthetic data are exchangeable across clients to augment local datasets. We theoretically and empirically show that the synthetic data can serve as \textbf{anchor points} to improves FL for two purposes: 1)
reducing the domain-gap between the distributions in the latent feature space; and 2) enabling global knowledge distillation for mutual learning. The effect of the two losses are visualized in Figure~\ref{fig:loss_effect}. In summary, we tackle a realistic and challenging setting in decentralized FL, where both data and model heterogeneities exist, without acquiring publicly available global real data. Our contributions are listed as follows:
\vspace{-2mm}
\begin{list}{\labelitemi}{\leftmargin=1em \itemindent=0em \itemsep=.2em}
  \setlength{\parskip}{0pt}
  \setlength{\parsep}{0pt}
\item To circumvent the heterogeneity on data and model, we propose an innovative algorithm named Decentralized Federated Learning with Synthetic Anchors (\ours{}) that utilizes only a small number of synthetic data. 
\item We theoretically and empirically show that the strategic design of synthetic anchor data and our novel FL loss function effectively boost local model generalization in diverse data scenarios.
\item We conduct extensive experiments prove \ours{}'s effectiveness, surpassing existing decentralized FL algorithms. It excels in inter- and intra-client performance across diverse tasks.
\end{list}
\vspace{-2mm}
\section{Preliminaries}

\subsection{Conventional Federated Learning}
\label{subsec: Conventional FL}


Conventional FL aims to learn a \textbf{single generalized global model} that performs optimally on all the clients' data domains. 
Mathematically, the learning problem can be formulated as 
\begin{equation}
    \label{Horizfl}
    M^* = \arg\min_{M\in \mathcal{M}} \sum_{i = 1}^{N} \mathbb{E}_{\x, y \sim P_i} [ \mathcal{L}(M(\x), y) ]
\end{equation}
where $M^*$ is the optimized global model from the shared model space $\mathcal{M}$, $P_i$ is the local data distribution at the $i$th client, $\mathcal{L}$ is the loss function and $(\x, y)$ is the feature-label pair.
Inspired by the pioneering work of FedAvg~\cite{mcmahan2017communication}, a plethora of methods have tried to fill in the performance gap of FedAvg on data-heterogeneous scenario, which can be categorized in two main orthogonal directions:  \emph{Direction 1} aims to minimize the difference between the local and global model parameters to improve convergence~\cite{li2020federated,karimireddy2020scaffold,wang2020tackling}. \emph{Direction 2} enforces consistency in local embedded features using anchors and regularization loss~\cite{tang2022virtual,zhou2022fedfa,ye2022fedfm}. 
This work follows the second research direction and aim to leverage anchor points to handle data heterogeneity. We also tackle the more challenging problem of domain shift, unlike other methods that only assume a label-shift amongst the client-data distributions.

\begin{table}[t]

\centering
\caption{\small Comparison of the settings with other related heterogeneous FL and decentralized FL methods.}
\label{tab:setting}
\begin{adjustbox}{max width=\linewidth}
\begin{threeparttable}
\begin{tabular}{c|c|c|c|c}
\toprule
Methods & \begin{tabular}[c]{@{}l@{}}Data Hetero-\\geneity\end{tabular} & \begin{tabular}[c]{@{}l@{}}Model Hetero-\\geneity\end{tabular} & Serverless & \begin{tabular}[c]{@{}l@{}}No Public\\ Data\end{tabular}
\\
\hline 
VHL~\cite{tang2022virtual}\tnote{a}  & \textcolor{red}{\cmark} & \textcolor{blue}{\xmark} & \textcolor{blue}{\xmark} & \textcolor{red}{\cmark}
\\
\hline 
FedGen~\cite{zhu2021data}  & \textcolor{red}{\cmark} & \textcolor{blue}{\xmark} & \textcolor{blue}{\xmark} & \textcolor{red}{\cmark}
\\
\hline
FedHe~\cite{chan2021fedhe} & \textcolor{blue}{\xmark} & \textcolor{red}{\cmark} & \textcolor{red}{\cmark} & \textcolor{red}{\cmark}
\\
\hline
FedDF~\cite{lin2020ensemble} & \textcolor{red}{\cmark} & \textcolor{red}{\cmark}  & \textcolor{blue}{\xmark} & \textcolor{blue}{\xmark}
\\
\hline
FCCL~\cite{huang2022learn}  & \textcolor{red}{\cmark} & \textcolor{red}{\cmark} & \textcolor{blue}{\xmark} & \textcolor{blue}{\xmark}
\\
\hline 
FedProto~\cite{tan2022fedproto}  & \textcolor{red}{\cmark} & \textcolor{red}{\cmark} & \textcolor{blue}{\xmark} & \textcolor{red}{\cmark}
\\
\hline 
FedFTG~\cite{zhang2022fine}  & \textcolor{red}{\cmark} & \textcolor{red}{\cmark} & \textcolor{blue}{\xmark} & \textcolor{red}{\cmark}
\\
\hline 
DENSE~\cite{zhang2022dense}  & \textcolor{red}{\cmark} & \textcolor{red}{\cmark} & \textcolor{blue}{\xmark} & \textcolor{red}{\cmark}
\\
\hline
\ours{} (ours) & \textcolor{red}{\textcolor{red}{\cmark}} & \textcolor{red}{\cmark} & \textcolor{red}{\cmark} & \textcolor{red}{\cmark} 
\\
\bottomrule
\end{tabular}%
\begin{tablenotes}
\small
    \item[a] VHL has a single global model, trained using mutual information from all clients. Therefore we reference it under Mutual Learning. 
\end{tablenotes}
\end{threeparttable}
\end{adjustbox}
\vspace{-5mm}
\end{table}

\subsection{Decentralized FL and Mutual Learning}
\label{sec:mutuallearning}
Standard decentralized FL aims to solve the similar generalization objective as conventional FL (\ie{} Eq.~\ref{Horizfl}), only, without a central server to do so~\cite{gao2022survey}, and the objective applies to each local models $M_i$. Here, we focus on learning from each other under heterogeneous models and data distributions. This brings in an essential line of works, known as \emph{mutual learning}, where clients learn from each other to obtain the essential generalizability for their models. For example, during quarantine for pandemic, hospitals want to collaboratively train a model for classifying the virus. It is desired the models are generalizable to virus variants from different area, so that after quarantine the local models are still effective for incoming tourists. 

Although mutual learning with heterogeneous data and models has been studied recently, most of them assume the existence of public real data~\cite{lin2020ensemble,huang2022learn,gong2022preserving} or a central server to coordinate the generation of synthetic data from the local client data~\cite{zhang2022dense,zhu2021data,zhang2022fine}. 
Another line of works rely on a server to aggregate locally generated logits or prototypes, and use it as local training guidance~\cite{jeong2018federated,chan2021fedhe,tan2022fedproto}. In addition, more recent works have suggested that each clients train two models, a larger model for local training and a smaller model for mutual information exchange~\cite{wu2022communication,shen2023federated}. However, none of the above methods simultaneously address both non-iid data and heterogeneous models under serverless and data-free setting. In this work, we explore mutual learning to optimize both local (intra-client) and global (inter-client) dataset accuracy (see the detailed setup in Sec.~\ref{sec:setting}).
We list the comparison with other methods in Table~\ref{tab:setting} and
more detailed related works in Appendix~\ref{app:related}.

\section{Method}

\subsection{Notation and Problem Setup}
\label{sec:setting}
Suppose there are $N$ clients with $i$th client denoted as $C_i$. 
Let's represent the \textbf{private datasets} on $C_i$ as $D_i=\{\x,y\}$,
where $\x$ is the feature and $y\in\{1,\cdots, K\}$ is the label from $K$ classes. Let $\mathcal{L}$ represent a real-valued loss function for classification (\eg cross-entropy loss).
Denote the communication neighboring nodes of the client $C_i$ in the system as $\mathcal{N}(C_i)$ and the local models as $\{ M_i  = \rho_i \circ \psi_i \}_{i = 1}^{i = N}$, where $\psi_i$ represents the feature encoder and $\rho_i$ represents the classification head for the $i$th client's model $M_i$. 
\ours{} returns trained client models $\{ M_i \}_{i = 1}^{i = N}$. 

Our work aims to connect two key areas,  heterogeneous FL and decentralized FL, termed as \underline{\textit{decentralized federated mutual learning}}, where we train \textbf{multiple client models} in a decentralized way such that they can \textbf{generalize} well across all clients' data domains. Mathematically, our objective is formulated as, for every client $i$,
\begin{align}\label{eq:fmd}
      M_i^* = & \argmin_{M_i \in \mathcal{M}_i}
    \underbrace{
    \ \mathbb{E}_{\x, y \sim P_i} [ \mathcal{L}(M_i(\x), y)]}_{\text{Intra-client}} \notag \\ &
    + \underbrace{\sum_{ j \in \mathcal{N}(C_i)} \mathbb{E}_{\x, y \sim P_j} [ \mathcal{L}(M_i(\x), y)]}_{\text{Inter-client}},
\end{align}
where $M_i^*$ is the best possible model for client $i$ with respect to the model space $\mathcal{M}_i$.


\begin{figure*}[t]
\centering
\subfloat[Heterogeneous Setup]{  
    \includegraphics[width=0.37\linewidth]{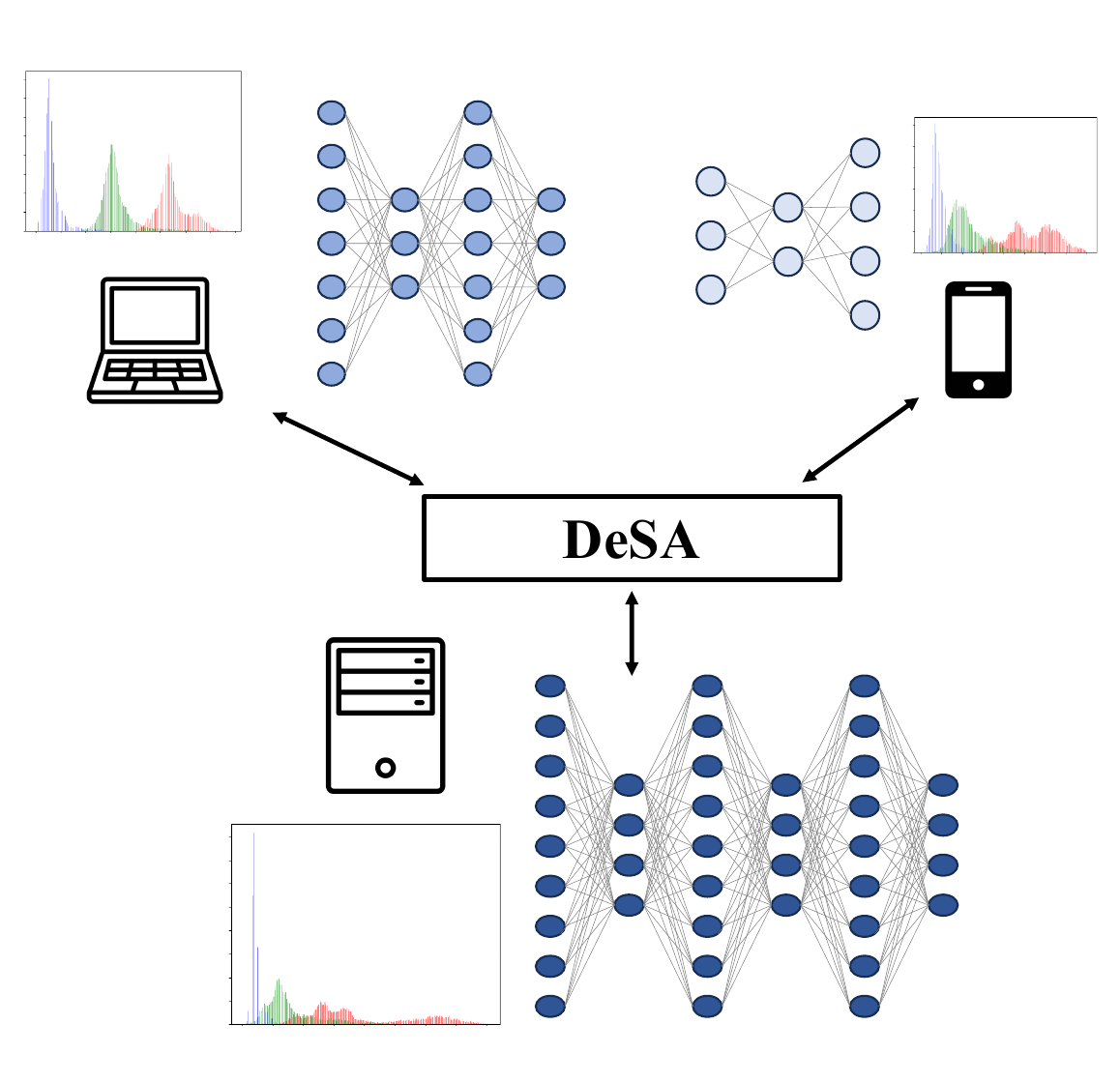}}
\subfloat[\ours{} pipeline]{
    \includegraphics[width=0.63\linewidth]{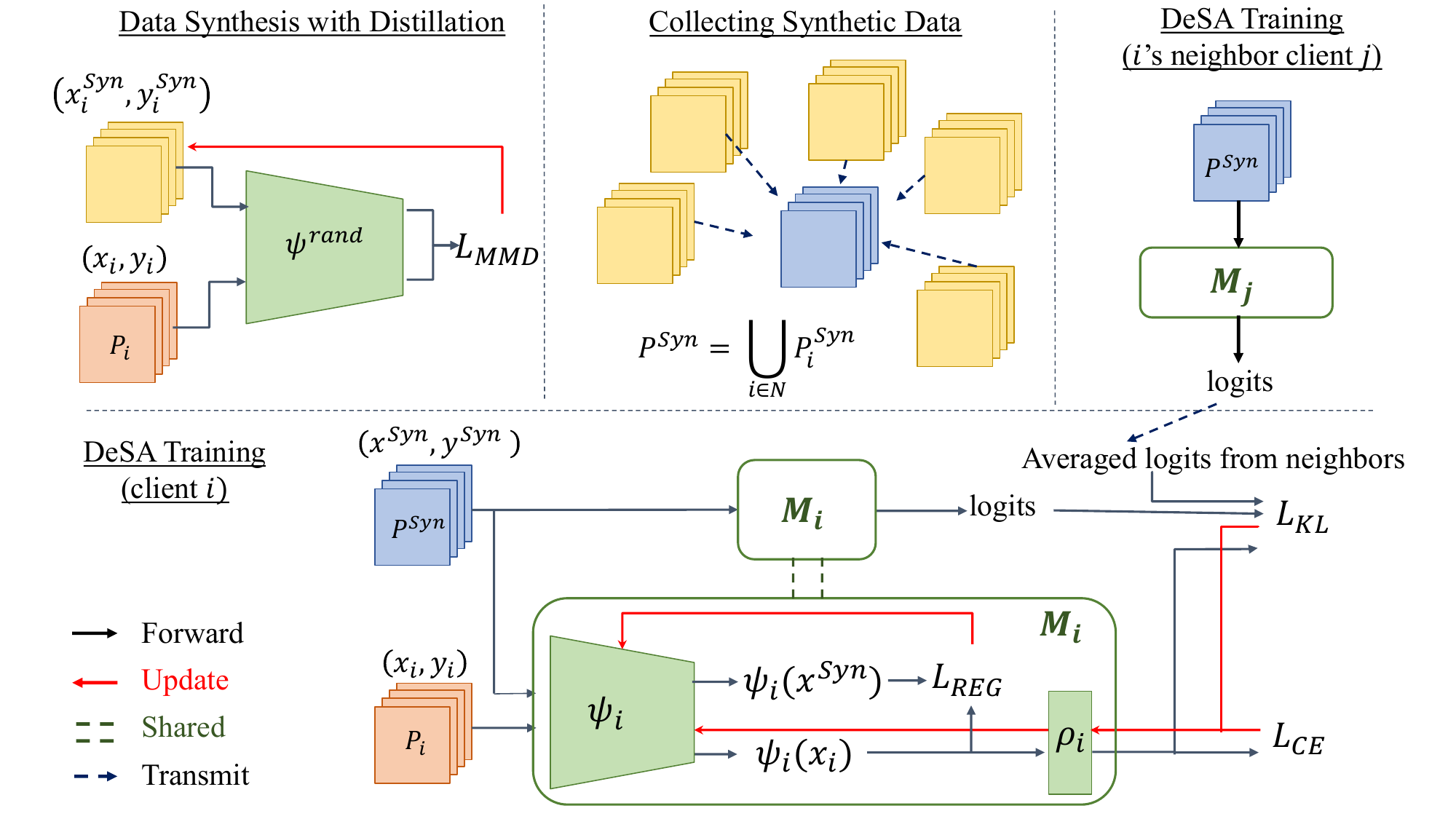}}
\caption{\small Heterogeneous setup and \ours{} pipeline. (a) We assume a realistic FL scenario, where clients have different data distributions and computational powers, which results in different model architectures. (b) \ours{} pipeline consists of three phases, local data synthesis (top left) , global synthetic data aggregation (top right)(Section ~\ref{sec:data_generation}), and decentralized training (bottom) using anchor regularization(Section ~\ref{sec:reg_loss})  and knowledge distillation (Section ~\ref{sec:kd_loss}).
}
\label{fig:main}
\vspace{-5mm}
\end{figure*}

\noindent \textbf{Overview of \ours{}.} The overall objective of \ours{} is to improve local models' generalizability in FL training under both model and data heterogeneity in a serverless setting as shown in in Figure~\ref{fig:main}(a). 
The pipeline of \ours{} is depicted in Figure~\ref{fig:main}(b). Our algorithm contains three important aspects: 1) we generate synthetic anchor data by matching raw data distribution and share them amongst the client's neighbors; 2) we train each client model locally with a synthetic anchor-based feature regularizer; and 3) we allow the models to learn from each other
via knowledge distillation based on the synthetic anchors.
The effectiveness of steps 2 and 3 can be observed in Figure~\ref{fig:loss_effect}.
The next three subsections delve deeper into these three designs. The full algorithm is depicted in Algorithm~\ref{alg:iterativeneigh}.

\subsection{Synthetic Anchor Datasets Generation}\label{sec:data_generation}
The recent success of dataset distillation-based data synthesis technique that generates data with similar representation power as the original raw data~\cite{zhao2020dataset,zhao2023dataset}. Thus, we propose to leverage this method to efficiently and effectively generate a synthetic anchor dataset without requiring any additional model pretraining. 
Inspired by our theoretical analysis in Sec.~\ref{sec:theory}, we utilize distribution matching~\cite{zhao2023dataset} to distill local synthetic anchor data using the empirical maximum mean discrepancy loss (MMD)~\cite{gretton2012kernel} as follows,
\begin{align}\label{eq:MMD}
D^{Syn}_i = \argmin_{D} ||& \frac{1}{|D_i|}\sum_{(\x,y) \in D_i} \psi^{\rm rand}(\x|y)\! 
\notag \\ &-\! \frac{1}{|D|}\sum_{(\x,y) \in D}\psi^{\rm rand}(\x|y)||^2,
\end{align} 
where $\psi^{\rm rand}$ is a randomly sampled feature extractor for each iteration, $D_i$ is the raw data for client $i$, and $D^{Syn}_i$ is its target synthetic data. Following Eq.~\ref{eq:MMD}, we can manipulate the synthetic anchor dataset generation in a class-balanced manner, which enables the label prior to being unbiased towards a set of classes.

Similar to other FL work sharing distilled synthetic data for efficiency~\cite{song2023federated}, 
After local data synthesis, we request each client to share it among peers to ensure they possess same global information, and
the global synthetic anchor data is denoted as $D^{Syn} = \cup_i D^{Syn}_i$~\footnote{By default, we perform simple interpolation (averaging) among clients as it is shown that using this mixup strategy can improve model fairness~\cite{chuang2021fair}.}. As shown in Algorithm~\ref{alg:iterativeneigh}, our method is designed to work only to receive neighbor node information, and DeSA is designed for peer-to-peer decentralized network. Since the loss requires all nodes’ information, we can leverage the FastMix algorithm to aggregate all nodes’ information as in \cite{ye2020decentralized,luo2022decentralized}. This method can aggregate all nodes' information via adjacent nodes’ communication at a linear speed. It is very common in fully decentralized optimization. In fact, our method can also work if each node can only receive neighbor nodes’ information during training, and we empirically show the feasibility in our CIFAR10C experiments by sampling neighboring clients. 

It is worth noting that, different from~\cite{song2023federated}, we further propose novel loss terms and training strategies to help mitigate the distribution discrepancy between the clients, which are detailed in the following sections (see Sec.~\ref{sec:reg_loss} and Sec.~\ref{sec:kd_loss}), enabling improved model performance, as intuitively incorporating $D^{Syn}$ into training can only achieve sub-optimal results (see Figure~\ref{fig:ablation_office} when both $\lambda_{KD}$ and $\lambda_{REG}$ equals to 0.). 

\subsection{REG Loss for Feature Regularization}\label{sec:reg_loss}
The synthetic anchor regularization loss enforces the model to learn a \textbf{client-domain invariant} representation of the data. 
\cite{tang2022virtual} and  other domain incremental works \cite{rostami2021lifelong} show that, adding a distribution discrepancy based loss in the \textbf{latent space} enables learning of a domain-invariant encoder $\psi$. However, most of the domain adaptation works require explicit access to the real data from other domains, or generates random noise as anchors. We propose using the latent space distribution of the synthetic anchor data $D^{Syn}$ as a synthetic anchor to which the client-model specific encoders $\psi_i$ can project their \textbf{local private data} onto. The loss function can be therefore defined as,
\begin{align}
    \label{eq:lsgd}
    \mathcal{L}_{REG}(\psi_i) =
    \mathbb{E}[d(\psi_i(D^{Syn})|| \psi_i(D_i))],
\end{align}
where $K$ is the number of classes, $d$ is distance computed using the supervised contrastive loss,
\begin{align}
    \label{eq:con}
    &d(\psi_i;D^{Syn}, D_i) = \sum_{j \in B}-\frac{1}{|B^{y_j}_{\backslash j}|} 
    \notag \\ &\sum_{\x_p\in B^{y_j}_{\backslash j}} {\rm log} \frac{{\rm exp}(\psi_i(\x_j) \cdot \psi_i(\x_p) / \tau_{temp})}{\sum_{\x_a \in B_{\backslash j}} {\rm exp}(\psi_i(\x_j) \cdot \psi_i(\x_a) / \tau_{temp})}
\end{align}
where $B_{\backslash j}$ represents a batch containing both local raw data $ D_i$ and global synthetic data $D^{Syn}$ but without data $j$, $B^{y_j}_{\backslash j}$ is a subset of $B_{\backslash j}$ only with samples belonging to class $y_{j}$, 
and $\tau_{temp}$ is a scalar temperature parameter. Note that we will detach the synthetic anchor data to ensure we are pulling local features to global features.

\subsection{Knowledge Distillation for Information Exchange}\label{sec:kd_loss}
This step allows a single client model to learn from all the other models using a common synthetic anchor dataset $D^{Syn}$. Under our setting of model heterogeneity among clients, we cannot aggregate the model parameters by simply averaging as in FedAvg~\cite{mcmahan2017communication}. Instead, we propose to utilize knowledge distillation (KD)~\cite{hinton2015distilling} for decentralized model aggregation. 
Specifically, the fact that $D^{Syn}$ is representative of the joint distributions of the clients allows it to be an ideal dataset for knowledge transfer. Thus, to enable the client model to mimic the predictions of the other models, we also incorporate KD loss using $D^{Syn}$, formulated as

\begin{align}
\label{eq:kd}
          &\mathcal{L}_{KD}(M_i) = \mathcal{L}_{KL} (M_i (\x^{Syn}) , \bar{Z}_i), 
          \notag \\ &\quad  \bar{Z}_i = \frac{1}{|N(C_i)|} \sum_{j \in N(C_i)} M_j (\x^{Syn}),
\end{align}
where $(\x^{Syn}, y^{Syn}) \sim D^{Syn}$, $N(i)$ is the neighbor clients of client $C_i$, and $\mathcal{L}_{KL}$ is the KL-Divergence between the output logits of $\x^{Syn}$ on $M_i$ and the averaged output logits of $\x^{Syn}$ on $M_j, \forall j \in N(C_i)$.
Finally, we formulate our objective function as
\begin{align}
\label{eq:glob}
          \mathcal{L} & = \mathcal{L}_{CE}(D_i \cup D^{Syn}; M_i) + \lambda_{REG} \mathcal{L}_{REG}(D_i, D^{Syn}; M_i) 
          \notag \\&
          + \lambda_{KD}\mathcal{L_{\rm KD}}(D^{Syn}; M_i, \bar{Z_i }),
\end{align}
where $\mathcal{L}_{CE}(D;M)$
is the $K$-classes  cross entropy loss on data $D$ and model $M$.
$\lambda_{REG}$ and $\lambda_{KD}$ are the hyperparameters for regularization and KD losses, $\mathcal{L}_{REG}$ and $\mathcal{L}_{KD}$ are as defined in Eq.~\ref{eq:lsgd} and Eq.~\ref{eq:kd}, and $\bar{Z_i}$ is the shared logits from neighboring clients $N(C_i)$.
We also incorporate our class-conditional-generated global synthetic data $D^{Syn}$ in the CE loss to enforce models to perform well on general domains and to benefit from the augmented dataset.
Overall, we formulate our objective function as

\begin{algorithm}[]
    \caption{ Serverless \ours{} (Procedures for Client $i$)}
    \label{alg:iterativeneigh}

    \begin{algorithmic}[1]

            
            

    \Procedure{Init}{$C_i$}
        \ForAll{$j \in \mathcal{N}(C_i)$}   
            \State $ D^{Syn} = D^{Syn} \cup \Call {Get-img}{C_j}$  
        \EndFor
    \EndProcedure{}
    
    \Procedure{LocalTrain}{$C_i$, $t$}
        \If {client $C_i$ is sampled} 
        \State share $Z_i = M_i (D^{Syn})$ to $\mathcal{N}(C_i)$     
        \State get $\bar{Z}_i = {1}/{|\mathcal{N}(C_i)|}\sum_j^{j \in \mathcal{N}(C_i)}Z_j$ 

            \State $\mathcal{L}_{CE}$ = \Call{classification}{$D_{i} \cup D^{Syn}; M_{i}$} 
        
            \State $\mathcal{L}_{REG}$ = \Call{Feature-REG}{$D_{i}, D^{Syn}; M_i$} 
            \State $\mathcal{L}_{KD}$ = \Call{KD}{$D^{Syn}; M_i, \bar{Z}$} 
            
            \State $\mathcal{L}$ = $\mathcal{L}_{CE} + \lambda_{REG}\mathcal{L}_{REG} + \lambda_{KD}\mathcal{L}_{KD}$
            \State $M_i = M_i  - \eta \nabla_{M_i} \mathcal{L}$ 
         \EndIf   
    \EndProcedure{}
    \end{algorithmic}

\end{algorithm}
\vspace{-6mm}  
\section{Theoretical Analysis}\label{sec:theory}
In this section, we focus on providing a theoretical justification for our algorithm.  The technical challenge is to analyze the effect of minimizing the overall loss function (Eq.~\ref{eq:glob}) on the generalizability on the global data distribution $P^T$ of a client model.\\
\textbf{Notation:} Here, we state the intuitive definitions borrowed from \cite{ben2010theory}. For the precise definitions please refer to the notation table in Appendix \ref{app:proof}. The domain pair  $(P, f^P)$ represents the source distribution and its optimal labeling function. The $\epsilon$ - error of a hypothesis $M$ on the domain pair $(P, f^P)$ is defined as the probability of a mismatch between the optimal labeling function $f^P$ and the hypothesis $M$. Additionally, the $\mathcal{H}\Delta\mathcal{H}$ divergence ($d_{\mathcal{H}\Delta \mathcal{H}}$) describes a distance measure between two distributions. 
\\
\textbf{Analysis:}
Our analysis focuses on the generalization on a global data distribution $P^T$ that is the average of the client distributions~\cite{marfoq2021federated}, with labeling function $f^T$ the same as $f^S$ .
We assume that the minimization of loss in Eq (\ref{eq:fmd}) matches the optimal labeling function $f^T$ of the global distribution $P^T$, and formalizes the intuition of generalizing over closely related client distributions. 

We proceed by defining the distribution of our global synthetic data as $P^{Syn}$ with the corresponding labeling function $f^{Syn}$.
As $P^{Syn}$ is also leveraged for the knowledge distillation, inspired by \cite{feng2021kd3a} we describe $(P_{KD}^{Syn}, f^{Syn}_{KD})$ as follows.

\begin{definition}\label{Def: Extended KD}
      The extended knowledge distillation (KD) domain pair ($P^{Syn}_{KD}$,$f^{Syn}_{KD}$) of a client $C_i$, originating from the KD dataset $D_{KD}^{Syn}$ is defined as $$D_{KD}^{Syn} \! = \! \{\x^{Syn}, \! \frac{1}{|N(C_i)|}\sum_{j \in N(C_i)}\!M_j(\x^{Syn})\} \!
      \sim \! (P_{KD}^{Syn}, f^{Syn}_{KD}) $$
      where $M_i(\x^{Syn})$ is the predicted logit on global synthetic data $\x^{Syn} \sim P^{Syn}(x)$. 
\end{definition}





\begin{definition}
    We define the the overall source distribution of the client $C_i$ as  $P_i^S$, which is a convex combination of the local and synthetic distributions
    \begin{align}
        P_i^S = \alpha P_i + \alpha^{Syn}P^{Syn} + \alpha^{Syn}_{KD} P^{Syn}_{KD} \label{Eq:Hull}
    \end{align}
\end{definition}
The positive component weights $\{\alpha, \alpha^{Syn}, \alpha^{Syn}_{KD}\}$ describe the dependence of the client $C_i$, on the local, synthetic and knowledge distillation data, and $\alpha + \alpha^{Syn}  + \alpha^{Syn}_{KD} = 1$. 
\begin{thm}\label{Thm:general_bound}
    Denote the client $C_i$'s model as $M_i = \rho_i \circ \psi_i \in {\mathcal{\boldsymbol{P}}}_i \circ \Psi_i = \mathcal{M}_i$ and its overall source distribution as $P_i^S$ with component weights ($\boldsymbol{\alpha}$). Then the generalization error on the global data distribution $P^T$ can be bounded as follows
    \begin{align}\label{eq:general_bound}
    \epsilon_{{P}^T}(M_i) &\le 
    \epsilon_{{P}_i^S}(M_i) + \alpha \mathbf{C}(P_i, P_T) 
    \notag \\ &+ \alpha^{Syn} \epsilon_{{P}^T}(f^{Syn})  + \alpha^{Syn}_{KD} \epsilon_{{P}^T}(f_{KD}^{Syn})  
    \notag \\ &+ \frac{(1 - \alpha)}{2} d_{\mathcal{P}_i \Delta \mathcal{P}_i}(\psi \circ P^{Syn}, \psi \circ P^T)
    \end{align}
where $\mathbf{C}(P_i, P_T)$ are small distance terms depending on the distributions $P_i$ and $P_T$.
\end{thm}

Below we give a short summary of the interpretation of our theorem in \ref{remark}.
For a more detailed interpretation, please refer to Appendix \ref{app:proof}.
\begin{remark} \label{remark} The first term is minimized via local cross-entropy loss. The second term is a constant given data distributions. The third and fourth terms measure the discrepancy between the labeling function of the target domain $f_T$ and the labeling function $f^{Syn}$ or $f^{Syn}_{KD}$ of the corresponding synthetic distribution $P^{Syn}$ or $P^{Syn}_{KD}$. With proper dataset distillation~\cite{zhao2023dataset},  we have the model trained on $D^{Syn}$ similar to that trained by $D^T$, \ie{}
    $\mathbb{E}_{\x \sim P^T} [l(M^{T}(\x),y)] \simeq \mathbb{E}_{\x \sim P^{Syn}} [l(M^{Syn}(\x),y)],$
implying $f^{Syn} \rightarrow f^T$ and then we have a small $\epsilon_{{P}^T}(f^{Syn})$. A small $\epsilon_{{P}^T}(f^{Syn}_{KD})$ can be achieved when every client achieves low $\mathcal{L}_{CE}$, indicating the model ability to learn $f_{KD}^{Syn}$ approximating to $f^T$. The last term $d_{\mathcal{P}_i \Delta \mathcal{P}_i}$  motivates the need for reducing our domain-invariant regularizer in Eq.~\ref{eq:lsgd}, elicited to be bounded.
\end{remark}
Furthermore, for the domain pair $(D, f^D)$, we denote $\mathcal{J}(D) = |\epsilon_D(M)  - \epsilon_{P^T}(M)| +\epsilon_{P^T}(f^D)$. The following proposition implies our generalization bound in Theorem~\ref{Thm:general_bound} is tighter than the generalization bound of training with local data in Eq.~\ref{eq:loss_on_Pi} \cite{ben2010theory} under some mild conditions. 
\begin{prop}\label{Prop:tighter_bound}
Under the conditions in Theorem~\ref{Thm:general_bound}, if it further holds that 
\begin{align}
\label{eq:tighter_bound}
\sup_{M\in\mathcal{M}_i} & \min \{ \mathcal{J}(P^{Syn}) , \mathcal{J}(P^{Syn}_{KD}) \}  \notag\\
     \le & \inf_{M\in\mathcal{M}_i}(\epsilon_{P_i}(M) - \epsilon_{{P}^T}(M)) + \mathbf{C}(P_i, P_T)
\end{align}
then we can get a tighter generalization bound on the $i$th client's model $M_i$ than learning with local data only.
\vspace{-2mm}
\end{prop}
The key idea of Proposition~\ref{Prop:tighter_bound} is to have the generalization bounds induced by $P^{Syn}$ and $P^{Syn}_{KD}$ is smaller than the generalization bound by the local training data distribution $P_i$.
When the local data heterogeneity is severe,  $\inf_{M\in\mathcal{M}_i}(\epsilon_{P_i}(M) - \epsilon_{{P}^T}(M))$ and $\mathbf{C}({P}_i ,{P}^{T})$ would be large.
As the synthetic data and the extended KD data are approaching the the global data distribution, the left side term in (\ref{eq:tighter_bound}) would be small.
Thus, the above proposition points out that, to reach better generalization, the model learning should rely more on the synthetic data and the extended KD data, when the local data are highly heterogeneous and the synthetic and the extended KD datasets are similar to the global ones.

\section{Experiment}

\begin{table*}[t!]
\centering
\caption{\small Heterogeneous model experiments. We compare with model heterogeneous FL methods and report the averaged global accuracy over all client models.The best accuracy is marked in \textbf{bold}. Observe that \ours{} can achieve best averaged global accuracy on \texttt{DIGITS} and \texttt{OFFICE}. For \texttt{CIFAR10C}, \ours{} can outperform most of the baseline methods except for FCCL, which utilizes CIFAR100 as the public dataset. This is because CIFAR100 and
\texttt{CIFAR10C} have a genuine semantic overlap.}
\label{tab:hetero_model}
\begin{adjustbox}{max width=\textwidth}
\begin{threeparttable}
\begin{tabular}{ll|llllll|lllll|ll}
\toprule
\multicolumn{2}{l|}{\multirow{2}{*}{}}                 & \multicolumn{6}{l|}{\texttt{DIGITS}}                                                                                                                             & \multicolumn{5}{l|}{\texttt{OFFICE}}                                                                                               & \multicolumn{2}{l}{\texttt{CIFAR10C}}       \\ \cline{3-15} 
\multicolumn{2}{l|}{}                                  & \multicolumn{1}{l|}{MN(C)\tnote{a}} & \multicolumn{1}{l|}{SV(A)} & \multicolumn{1}{l|}{US(C)} & \multicolumn{1}{l|}{Syn(A)} & \multicolumn{1}{l|}{MM(C)} & Avg   & \multicolumn{1}{l|}{AM(A)} & \multicolumn{1}{l|}{CA(C)} & \multicolumn{1}{l|}{DS(A)} & \multicolumn{1}{l|}{WE(A)} & Avg   & \multicolumn{1}{l|}{0.1}   & 0.2   \\ \hline
\multicolumn{2}{l|}{FedHe}                             & \multicolumn{1}{l|}{59.51} & \multicolumn{1}{l|}{66.67} & \multicolumn{1}{l|}{49.89} & \multicolumn{1}{l|}{75.39}  & \multicolumn{1}{l|}{71.57} & 64.81 & \multicolumn{1}{l|}{33.33} & \multicolumn{1}{l|}{47.17} & \multicolumn{1}{l|}{36.86} & \multicolumn{1}{l|}{52.96} & 42.59 & \multicolumn{1}{l|}{47.73} &  51.26     \\ \hline
\multicolumn{1}{l|}{\multirow{2}{*}{FedDF}} & Cifar100 & \multicolumn{1}{l|}{65.98}     & \multicolumn{1}{l|}{65.21}     & \multicolumn{1}{l|}{61.30}     & \multicolumn{1}{l|}{69.65}      & \multicolumn{1}{l|}{\textbf{74.48}}     & 67.32     & \multicolumn{1}{l|}{38.87} & \multicolumn{1}{l|}{49.51} & \multicolumn{1}{l|}{33.12} & \multicolumn{1}{l|}{46.89} & 42.09 & \multicolumn{1}{l|}{27.69} & 35.70 \\ \cline{2-15}
\multicolumn{1}{l|}{}                       & FMNIST   & \multicolumn{1}{l|}{43.05}     & \multicolumn{1}{l|}{69.14}     & \multicolumn{1}{l|}{44.95}     & \multicolumn{1}{l|}{74.67}      & \multicolumn{1}{l|}{71.27}     & 60.61     & \multicolumn{1}{l|}{39.13} & \multicolumn{1}{l|}{46.53} & \multicolumn{1}{l|}{40.23} & \multicolumn{1}{l|}{43.77} & 42.36 & \multicolumn{1}{l|}{28.26} & 36.50 \\ \hline 
\multicolumn{1}{l|}{\multirow{2}{*}{FCCL}}  & Cifar100 & \multicolumn{1}{l|}{-}     & \multicolumn{1}{l|}{-}     & \multicolumn{1}{l|}{-}     & \multicolumn{1}{l|}{-}      & \multicolumn{1}{l|}{-}     & -     & \multicolumn{1}{l|}{38.22} & \multicolumn{1}{l|}{49.10} & \multicolumn{1}{l|}{44.68} & \multicolumn{1}{l|}{52.26} & 46.07 & \multicolumn{1}{l|}{51.70} & 50.78 \\ \cline{2-15} 
\multicolumn{1}{l|}{}                       & FMNIST   & \multicolumn{1}{l|}{46.43}     & \multicolumn{1}{l|}{61.02}     & \multicolumn{1}{l|}{42.64}     & \multicolumn{1}{l|}{63.05}      & \multicolumn{1}{l|}{66.39}     & 55.91     & \multicolumn{1}{l|}{27.39} & \multicolumn{1}{l|}{46.78} & \multicolumn{1}{l|}{38.56} & \multicolumn{1}{l|}{48.47} & 40.30 & \multicolumn{1}{l|}{\textbf{52.40}}     & 51.83     \\ \hline 
\multicolumn{2}{l|}{FedProto}                          & \multicolumn{1}{l|}{62.59} & \multicolumn{1}{l|}{71.74} & \multicolumn{1}{l|}{58.52} & \multicolumn{1}{l|}{\textbf{81.19}}  & \multicolumn{1}{l|}{74.44} & 69.70 & \multicolumn{1}{l|}{38.08} & \multicolumn{1}{l|}{25.06} & \multicolumn{1}{l|}{26.49} & \multicolumn{1}{l|}{47.22} & 34.21 & \multicolumn{1}{l|}{16.82} & 31.39 \\ \hline 
\multicolumn{2}{l|}{\ours($D_{\rm VHL}^{Syn}$)\tnote{b}}                              & \multicolumn{1}{l|}{54.40} & \multicolumn{1}{l|}{62.03} & \multicolumn{1}{l|}{42.34} & \multicolumn{1}{l|}{67.75}  & \multicolumn{1}{l|}{73.03} & 59.91 & \multicolumn{1}{l|}{8.82}  & \multicolumn{1}{l|}{48.98} & \multicolumn{1}{l|}{16.90} & \multicolumn{1}{l|}{49.13} & 30.96 & \multicolumn{1}{l|}{47.49} & 52.04 \\ \hline 
\multicolumn{2}{l|}{\ours}                              & \multicolumn{1}{l|}{\textbf{70.12}} & \multicolumn{1}{l|}{\textbf{76.17}} & \multicolumn{1}{l|}{\textbf{71.17}} & \multicolumn{1}{l|}{81.10}  & \multicolumn{1}{l|}{73.83} & \textbf{74.47} & \multicolumn{1}{l|}{\textbf{51.35}} & \multicolumn{1}{l|}{\textbf{52.80}} & \multicolumn{1}{l|}{\textbf{52.17}} & \multicolumn{1}{l|}{\textbf{52.31}} & \textbf{54.46} & \multicolumn{1}{l|}{48.19} & \textbf{52.80} \\ \bottomrule 
\end{tabular}
\begin{tablenotes} 
    \item[a] The letter inside the parenthesis is the model architecture used by the client. A and C represent AlexNet and ConvNet, respectively.
    \item[b] For VHL baseline, we use the synthetic data sampling strategy in VHL only. The purpose is to show \ours{} can generate better synthetic anchor data for feature regularization and knowledge distillation.
\end{tablenotes}
\end{threeparttable}
\end{adjustbox}
\vspace{-5mm}
\end{table*}

\subsection{Training Setup}\label{exp:setup}
\noindent\textbf{Datasets and Models}
We extensively evaluate \ours{} under data heterogeneity in our experiments. Specifically, we consider three classification tasks on three sets of domain-shifted datasets:\\
\noindent1) \texttt{DIGITS}=\{MNIST~\cite{lecun1998gradient}, SVHN~\cite{netzer2011reading}, USPS~\cite{hull1994database}, SynthDigits~\cite{ganin2015unsupervised}, MNIST-M~\cite{ganin2015unsupervised}\} contains digits from different styles, and each dataset represents one client.\\
\noindent2) \texttt{OFFICE}=\{Amazon~\cite{saenko2010adapting}, Caltech~\cite{griffin2007caltech}, DSLR~\cite{saenko2010adapting}, and WebCam~\cite{saenko2010adapting}\} contains images from different cameras and environments, and, similarly, each dataset represents one client.\\
\noindent 3) \texttt{CIFAR10C} consists 57 subsets with domain- and label-shifted datasets sampled using Dirichlet distribution with $\beta=2$ from Cifar10-C~\cite{hendrycks2019benchmarking}.\\
More information about datasets and image synthesis can be found in Appendix~\ref{app:synthetic_images}. 
In our model heterogeneity experiments (Sec.~\ref{exp:hetero_model}), we randomly assign model architectures from \{ConvNet, AlexNet\} for each client, while in model homogeneous experiments, we use ConvNet for all clients (see Appendix~\ref{app:model_arch} for model details).

\noindent\textbf{Comparison Methods} 
We compare \ours{} with two sets of baseline federated learning methods: one considers heterogeneous models (Sec.~\ref{exp:hetero_model}) and the other considers homogeneous models (Sec.~\ref{exp:homo_model}). For \emph{\textbf{heterogeneous model experiments}}, we compare with FedHe~\cite{chan2021fedhe}, FedDF~\cite{lin2020ensemble}, FCCL~\cite{huang2022learn}, and FedProto~\cite{tan2022fedproto}, and assume the clients owns personalized models\footnote{For the purposes of this comparison, we have excluded FedFTG~\cite{zhang2022fine} and DENSE~\cite{zhang2022dense}, which address heterogeneities in different learning scenarios. FedFTG focuses on fine-tuning a global model, and DENSE belongs to one-shot FL, and both of them requires aggregate local information and train a generator on the server side. Note that none of the data-sharing-based baseline methods employ privacy-preserving techniques.}. For  \emph{\textbf{homogeneous model experiments}}, we compare with FedAvg~\cite{mcmahan2017communication}, FedProx~\cite{li2019convergence}, 
MOON~\cite{li2021model},
Scaffold~\cite{karimireddy2020scaffold}, FedGen~\cite{zhu2021data}, and VHL~\cite{tang2022virtual}, and assume these baseline methods can leverage a server for global model aggregation. 

\noindent\textbf{FL Training Setup} 
If not otherwise specified, we use SGD optimizer with a learning rate of $10^{-2}$, and our default setting for local model update epochs is 1, total update rounds is 100, and the batch size for local training is 32. Since we only have a few clients for \texttt{DIGITS} and \texttt{OFFICE} experiments, we will select all the clients for each iteration, while we randomly sample $10\%$ and $20\%$ clients for each round when performing \texttt{CIFAR10C} experiments. By default, $\lambda_{\rm REG}$ and $\lambda_{\rm KD}$ are set to 1. 
 
\subsection{Heterogeneous Model Experiments}\label{exp:hetero_model}

The objective of the experiments is to show that \ours{} can effectively leverage and learn generalized information from other clients under data and model heterogeneities. Thus, we report the averaged global accuracy by testing $i$-th local model on every client $j$'s ($\forall j, j \in [N]$) test sets.
Note FedDF and FCCL require accessing to public available data. To make a fair comparison, we use FMNIST~\cite{xiao2017fashion} and Cifar100~\cite{krizhevsky2009learning} as public datasets for knowledge distillation. \ours{}($D_{\rm VHL}^{Syn}$) uses our training pipeline, but the synthetic data is sampled from an untrained StyleGAN~\cite{karras2019style} as in VHL~\cite{tang2022virtual}. 

Based on our experiment results for heterogeneous models in Table.~\ref{tab:hetero_model}, it is evident that our method \ours{} significantly enhances the overall accuracy on \texttt{DIGITS} and \texttt{OFFICE}. In the \texttt{CIFAR10C} experiments, we increase the total training rounds to 200 since the performance of FL is notably hindered by a low client sampling ratio, especially in the case of serverless methods. As shown in the table, \ours{} can improve the global accuracy and out-performs most of the baseline methods, except for FCCL when the client sampling ratio is 0.1. We believe this is due to the extreme decentralized learning setting, \emph{i.e.,} under model and data heterogeneity, serverless, and low client sampling ratio scenarios. FCCL utilizes pre-trained models, which provides a good starting point. However, we observe that the method tends to overfit easier by the fact that the performance drops when the client sampling ratio increases. Moreover, it is worth mentioning that the accuracies of FedDF and FCCL vary significantly when switching from Cifar100 to FMNIST (the FedDF-\texttt{DIGITS} and FCCL-\texttt{OFFICE} results). We found that the training with both methods are easy to obtain \texttt{nan} loss, which we mark as `-' in the table. This suggests that these methods strongly rely on the public data, which restricts the utility of the methods. In contrast, \ours{} does not depend on public data, and furthermore, it is completely serverless.

\subsection{Homogeneous Model Experiments}\label{exp:homo_model}

\begin{table}[th]
\centering
\caption{\small Homogeneous model experiments. We compare with model homogeneous FL methods and report the averaged local accuracy. The best accuracy is marked in \textbf{bold}.}
\resizebox{0.85\linewidth}{!}{
\begin{tabular}{l|l|l|ll}
\toprule
\multirow{2}{*}{} & \multirow{2}{*}{\texttt{DIGITS}} & \multirow{2}{*}{\texttt{OFFICE}} & \multicolumn{2}{l}{\texttt{CIFAR10C}}      \\ \cline{4-5} 
                  &                         &                         & \multicolumn{1}{l|}{0.1}   & 0.2   \\ \hline
FedAvg            & 94.20                   & 76.45                   & \multicolumn{1}{l|}{65.26} & 66.40 \\ \hline
FedProx           & 94.19                   & 76.45                   & \multicolumn{1}{l|}{65.33} & 66.36 \\ \hline
MOON              & 94.37                   & 73.64                   & \multicolumn{1}{l|}{64.74} & 66.70 \\ \hline
Scaffold          & 94.95                   & 77.52                   & \multicolumn{1}{l|}{\textbf{65.66}} & 67.15 \\ \hline
VHL               & 94.11                   & 75.69                   & \multicolumn{1}{l|}{64.67} & 66.55 \\ \hline
FedGen            & 82.62                   & 63.60                   & \multicolumn{1}{l|}{45.77} & 48.10 \\ \hline
\ours{}              & \textbf{95.53}                   & \textbf{82.92}                   & \multicolumn{1}{l|}{64.47} & \textbf{68.13} \\ \bottomrule
\end{tabular}
}
\label{tab:homo_exp}
\vspace{-3mm}
\end{table}

Among the baseline methods for homogeneous model experiments, they learn a generalizable global model via model aggregation on a server, which is not required in \ours{}. As these baseline methods are only evaluated on their single global model, for a fair comparison, we report the averaged local accuracy in the experiments. Specifically, the average local accuracy for the baseline methods is the average performance over testing the global model over all the clients; while we report, we report the average local accuracy on testing $i$-th local model on client $i$'s test set over all the clients. 

One can observe from Table~\ref{tab:homo_exp} that \ours{} can effectively leverage local information and outperforms other methods \texttt{DIGITS} and \texttt{OFFICE}. For \texttt{CIFAR10C}, although \ours{} has highest averaged local accuracy when client ratio is 0.2, it has lower performance when client ratio is 0.1. This is because smaller client sampling ratios have a larger impact on decentralized learning as we do not have a global model, and thus some clients may suffer from low model performance due to insufficient training and the scarce global information from the sampled neighbor clients. Overall, despite the serverless setting, \ours{} is compatible with the baseline methods that have central servers.

\subsection{Ablation studies for \ours{}}\label{exp:ablation}

The effectiveness of \ours{} relies on the novel designs of synthetic anchor data and the losses. To evaluate how these designs influences the performance of \ours{}, we vary the number of synthetic anchor data (IPC) and the loss coefficients ($\lambda's$) in the following paragraphs. If not otherwise specified, we use the default hyperparameters and model heterogeneous setting. We report the averaged global accuracy in this section.

\begin{figure}[th]
\centering
\vspace{-3mm}
\subfloat[$\lambda_{KD}$]{  
    \includegraphics[width=0.48\linewidth]{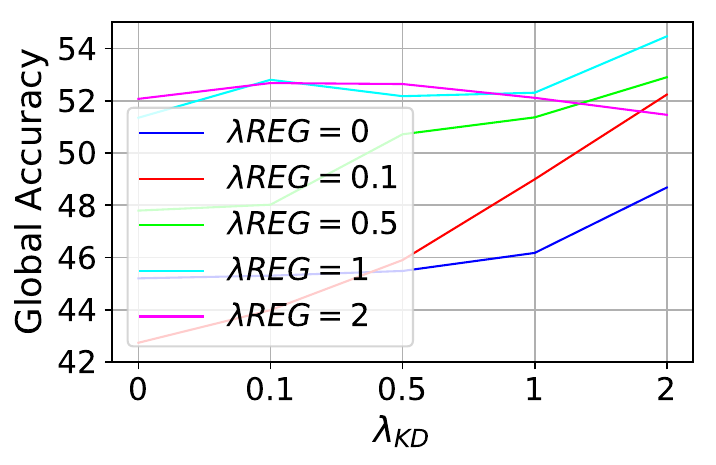}}
\subfloat[$\lambda_{REG}$]{
    \includegraphics[width=0.48\linewidth]{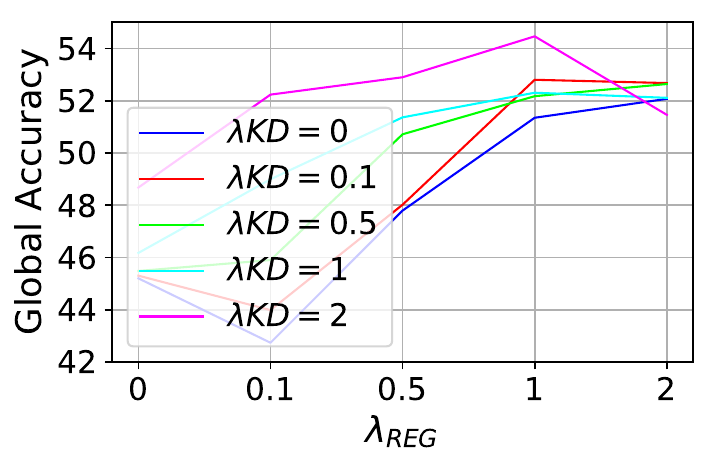}}
\caption{\small Ablation studies for $\lambda$'s using \texttt{OFFICE}. We report the averaged global accuracy when changing the $\lambda$ values.}
\label{fig:ablation_office}
\vspace{-3mm}
\end{figure}

\begin{table}[th]
\centering
\caption{\small Ablation study on the size of synthetic dataset using \texttt{DIGITS}. We use Images-Per-Class (IPC) to show its size.}
\resizebox{\linewidth}{!}{
\begin{tabular}{l|l|l|l|l|l|l}
\toprule
IPC        & 5     & 10    & 20    & 50    & 100    & 200   \\ \hline
Global Acc & 70.12 & 72.74 & 72.46 & 74.32 & 70.29 & 70.45 \\ \bottomrule
\end{tabular}
}
\label{tab:ipc}
\vspace{-3mm}
\end{table}

\noindent\textbf{Evaluation of $\lambda$ Selections}
$\lambda_{KD}$ and $\lambda_{REG}$ play an important role to help the clients to learn generalized information as well as improving the performance on local data. We select \texttt{OFFICE} as the candidate dataset for this set of experiments because it has larger domain-shift. We vary both of $\lambda_{KD}$ and $\lambda_{REG}$ between 0 to 2 and report the global accuracy in Figure~\ref{fig:ablation_office}(a). One can observe that when $\lambda_{KD}$ or $\lambda_{REG}$ increases, the overall global accuracy increases. However, in Figure~\ref{fig:ablation_office}(b), when we increase $\lambda_{REG}$ to 2, the performance drops. This happens because the magnitude of REG loss term dominates the total training loss. Overall, we conclude that both $\lambda_{KD}$ helps the local model learn information from other clients' models, and $\lambda_{REG}$ improves the global performance by enforcing the local model to learn generalized features.

\noindent\textbf{Evaluation of Size of Synthetic Dataset}
The size of synthetic data is a critical hyperparameter for \ours{} as it represents the shared local information. Since \ours{} synthesizes class-balanced data, we use Images-Per-Class (IPC) to represent the size of the synthetic data. We select \texttt{DIGITS} as the candidate dataset for this set of experiments because it contains larger number of data for each client, which allows us to increase IPC up to 200. One can observe in Table~\ref{tab:ipc} that blindly increasing the IPC does not guarantee to obtain optimal global accuracy. It will cause the loss function to be dominated by the last 2 terms of Eq.~\ref{eq:general_bound}, \textit{i.e.,} by synthetic data. However, synthesizing larger number of synthetic data may degrade its quality, and the sampled batch for $\mathcal{L}_{REG}$ may fail to capture the distribution.

\subsection{Further Discussion}\label{exp:further}

\noindent\textbf{Communication Overhead}
Although \ours{} requires training and transferring local synthetic data to every clients, the process happens before the FL training, and can be pre-processed offline.
During \ours{} training, clients only need to share logits \textit{w.r.t.} global synthetic data, resulting in a lightweight \textit{in-training} communication overhead. We discuss the communication of \ours{} compared to standard FL in Appendix~\ref{app:communication}, and show that \ours{} has lower overall communication overhead.

\noindent\textbf{Privacy}
\ours{} requires sharing image-level information among clients in FL, which may raise privacy concerns. Therefore, we empirically show that our distilled synthetic data can protect against some privacy attacks in Appendix~\ref{app:MIA}. We further discuss \ours{}'s potential for higher privacy guarantee using Differential Privacy~\cite{abadi2016deep} in Appendix~\ref{app:DP}.
Furthermore, we claim that decentralized FL with both data and model heterogeneities is an extremely challenging setting, where existing solutions either require sharing real public data~\cite{lin2020ensemble,huang2022learn} or synthetic data generated from GAN-based generator~\cite{zhang2022dense,zhang2022fine}.

\noindent\textbf{Theory vs. Practice}
We note that obtaining the tight bound from our theoretical findings in Theorem~\ref{Thm:general_bound} requires proper dataset distillation. 
In Section ~\ref{sec:data_generation}, we propose an efficient approximate solution for dataset distillation. Although perfect dataset distillation may not be achievable (as shown in the visualization in Appendix~\ref{app:synthetic_images}), we have found through experimentation that using our synthetic data in combination with the proposed $\mathcal{L}_{REG}$ (Eq.~\ref{eq:lsgd}) and $\mathcal{L}_{KD}$ (Eq.~\ref{eq:kd}) can already lead to improved generalization.

\section{Conclusion}
A novel and effective method, \ours{}, is presented that utilizes synthetic data to deal with both data and model heterogeneities in serverless decentralized FL. In particular, \ours{} introduces a pipeline that involves synthetic data generalization, and we propose a new scheme that incorporates the synthetic data as anchor points in decentralized FL model training. To address heterogeneity issues, we utilize the synthetic anchor data and propose two regularization losses: anchor loss and knowledge distillation loss. We provide theoretical analysis on the generalization bound to justify the effectiveness of \ours{} using the synthetic anchor data. Empirically, the resulted client models not only achieve compelling local performance but also can generalize well onto other clients' data distributions, boosting cross-domain performance. Through extensive experiments on various classification tasks, we show that \ours{} robustly improves the efficacy of collaborative learning when compared with state-of-the-art methods, under both model and data heterogeneous settings.
\section*{Acknowledgement}
This work is supported in part by the Natural Sciences and Engineering Research Council of Canada (NSERC), Public Safety Canada, CIFAR Catalyst Grant, and Compute Canada Research Platform.
\section*{Impact Statement}
This paper represents an effort to progress the field of machine learning and distributed learning. Our work has various potential societal impacts, although we do not believe any specific one must be highlighted here.


\bibliography{ref}

\begin{thebibliography}{81}
\providecommand{\natexlab}[1]{#1}
\providecommand{\url}[1]{\texttt{#1}}
\expandafter\ifx\csname urlstyle\endcsname\relax
  \providecommand{\doi}[1]{doi: #1}\else
  \providecommand{\doi}{doi: \begingroup \urlstyle{rm}\Url}\fi

\bibitem[Abadi et~al.(2016)Abadi, Chu, Goodfellow, McMahan, Mironov, Talwar, and Zhang]{abadi2016deep}
Abadi, M., Chu, A., Goodfellow, I., McMahan, H.~B., Mironov, I., Talwar, K., and Zhang, L.
\newblock Deep learning with differential privacy.
\newblock In \emph{Proceedings of the 2016 ACM SIGSAC conference on computer and communications security}, pp.\  308--318, 2016.

\bibitem[Albuquerque et~al.(2019)Albuquerque, Monteiro, Darvishi, Falk, and Mitliagkas]{albuquerque2019generalizing}
Albuquerque, I., Monteiro, J., Darvishi, M., Falk, T.~H., and Mitliagkas, I.
\newblock Generalizing to unseen domains via distribution matching.
\newblock \emph{arXiv preprint arXiv:1911.00804}, 2019.

\bibitem[Assran et~al.(2019)Assran, Loizou, Ballas, and Rabbat]{assran2019stochastic}
Assran, M., Loizou, N., Ballas, N., and Rabbat, M.
\newblock Stochastic gradient push for distributed deep learning.
\newblock In \emph{International Conference on Machine Learning}, pp.\  344--353. PMLR, 2019.

\bibitem[Beltr{\'a}n et~al.(2023)Beltr{\'a}n, P{\'e}rez, S{\'a}nchez, Bernal, Bovet, P{\'e}rez, P{\'e}rez, and Celdr{\'a}n]{beltran2023decentralized}
Beltr{\'a}n, E. T.~M., P{\'e}rez, M.~Q., S{\'a}nchez, P. M.~S., Bernal, S.~L., Bovet, G., P{\'e}rez, M.~G., P{\'e}rez, G.~M., and Celdr{\'a}n, A.~H.
\newblock Decentralized federated learning: Fundamentals, state of the art, frameworks, trends, and challenges.
\newblock \emph{IEEE Communications Surveys \& Tutorials}, 2023.

\bibitem[Ben-David et~al.(2006)Ben-David, Blitzer, Crammer, and Pereira]{ben2006analysis}
Ben-David, S., Blitzer, J., Crammer, K., and Pereira, F.
\newblock Analysis of representations for domain adaptation.
\newblock \emph{Advances in neural information processing systems}, 19, 2006.

\bibitem[Ben-David et~al.(2010)Ben-David, Blitzer, Crammer, Kulesza, Pereira, and Vaughan]{ben2010theory}
Ben-David, S., Blitzer, J., Crammer, K., Kulesza, A., Pereira, F., and Vaughan, J.~W.
\newblock A theory of learning from different domains.
\newblock \emph{Machine learning}, 79:\penalty0 151--175, 2010.

\bibitem[Carlini et~al.(2022{\natexlab{a}})Carlini, Chien, Nasr, Song, Terzis, and Tramer]{carlini2022membership}
Carlini, N., Chien, S., Nasr, M., Song, S., Terzis, A., and Tramer, F.
\newblock Membership inference attacks from first principles.
\newblock In \emph{2022 IEEE Symposium on Security and Privacy (SP)}, pp.\  1897--1914. IEEE, 2022{\natexlab{a}}.

\bibitem[Carlini et~al.(2022{\natexlab{b}})Carlini, Feldman, and Nasr]{carlini2022no}
Carlini, N., Feldman, V., and Nasr, M.
\newblock No free lunch in" privacy for free: How does dataset condensation help privacy".
\newblock \emph{arXiv preprint arXiv:2209.14987}, 2022{\natexlab{b}}.

\bibitem[Cazenavette et~al.(2022)Cazenavette, Wang, Torralba, Efros, and Zhu]{cazenavette2022dataset}
Cazenavette, G., Wang, T., Torralba, A., Efros, A.~A., and Zhu, J.-Y.
\newblock Dataset distillation by matching training trajectories.
\newblock In \emph{Proceedings of the IEEE/CVF Conference on Computer Vision and Pattern Recognition}, pp.\  4750--4759, 2022.

\bibitem[Chan \& Ngai(2021)Chan and Ngai]{chan2021fedhe}
Chan, Y.~H. and Ngai, E.~C.
\newblock Fedhe: Heterogeneous models and communication-efficient federated learning.
\newblock In \emph{2021 17th International Conference on Mobility, Sensing and Networking (MSN)}, pp.\  207--214. IEEE, 2021.

\bibitem[Chang et~al.(2018)Chang, Balachandar, Lam, Yi, Brown, Beers, Rosen, Rubin, and Kalpathy-Cramer]{chang2018distributed}
Chang, K., Balachandar, N., Lam, C., Yi, D., Brown, J., Beers, A., Rosen, B., Rubin, D.~L., and Kalpathy-Cramer, J.
\newblock Distributed deep learning networks among institutions for medical imaging.
\newblock \emph{Journal of the American Medical Informatics Association}, 25\penalty0 (8):\penalty0 945--954, 2018.

\bibitem[Chuang \& Mroueh(2021)Chuang and Mroueh]{chuang2021fair}
Chuang, C.-Y. and Mroueh, Y.
\newblock Fair mixup: Fairness via interpolation.
\newblock \emph{arXiv preprint arXiv:2103.06503}, 2021.

\bibitem[Crammer et~al.(2008)Crammer, Kearns, and Wortman]{crammer2008learning}
Crammer, K., Kearns, M., and Wortman, J.
\newblock Learning from multiple sources.
\newblock \emph{Journal of Machine Learning Research}, 9\penalty0 (8), 2008.

\bibitem[Cui et~al.(2023)Cui, Wang, Si, and Hsieh]{cui2023scaling}
Cui, J., Wang, R., Si, S., and Hsieh, C.-J.
\newblock Scaling up dataset distillation to imagenet-1k with constant memory.
\newblock In \emph{International Conference on Machine Learning}, pp.\  6565--6590. PMLR, 2023.

\bibitem[Donahue \& Kleinberg(2021)Donahue and Kleinberg]{donahue2021model}
Donahue, K. and Kleinberg, J.
\newblock Model-sharing games: Analyzing federated learning under voluntary participation.
\newblock In \emph{Proceedings of the AAAI Conference on Artificial Intelligence}, volume~35, pp.\  5303--5311, 2021.

\bibitem[Dong et~al.(2022)Dong, Zhao, and Lyu]{dong2022privacy}
Dong, T., Zhao, B., and Lyu, L.
\newblock Privacy for free: How does dataset condensation help privacy?
\newblock \emph{arXiv preprint arXiv:2206.00240}, 2022.

\bibitem[Fallah et~al.(2020)Fallah, Mokhtari, and Ozdaglar]{fallah2020personalized}
Fallah, A., Mokhtari, A., and Ozdaglar, A.
\newblock Personalized federated learning with theoretical guarantees: A model-agnostic meta-learning approach.
\newblock \emph{Advances in Neural Information Processing Systems}, 33:\penalty0 3557--3568, 2020.

\bibitem[Feng et~al.(2021)Feng, You, Chen, Zhang, Zhu, Wu, Wu, and Chen]{feng2021kd3a}
Feng, H., You, Z., Chen, M., Zhang, T., Zhu, M., Wu, F., Wu, C., and Chen, W.
\newblock Kd3a: Unsupervised multi-source decentralized domain adaptation via knowledge distillation.
\newblock In \emph{ICML}, pp.\  3274--3283, 2021.

\bibitem[Ganin \& Lempitsky(2015)Ganin and Lempitsky]{ganin2015unsupervised}
Ganin, Y. and Lempitsky, V.
\newblock Unsupervised domain adaptation by backpropagation.
\newblock In \emph{International conference on machine learning}, pp.\  1180--1189. PMLR, 2015.

\bibitem[Gao et~al.(2022)Gao, Yao, and Yang]{gao2022survey}
Gao, D., Yao, X., and Yang, Q.
\newblock A survey on heterogeneous federated learning.
\newblock \emph{arXiv preprint arXiv:2210.04505}, 2022.

\bibitem[Ghosh et~al.(2022)Ghosh, Chung, Yin, and Ramchandran]{ghosh2022efficient}
Ghosh, A., Chung, J., Yin, D., and Ramchandran, K.
\newblock An efficient framework for clustered federated learning.
\newblock \emph{IEEE Transactions on Information Theory}, 68\penalty0 (12):\penalty0 8076--8091, 2022.

\bibitem[Gong et~al.(2022)Gong, Sharma, Karanam, Wu, Chen, Doermann, and Innanje]{gong2022preserving}
Gong, X., Sharma, A., Karanam, S., Wu, Z., Chen, T., Doermann, D., and Innanje, A.
\newblock Preserving privacy in federated learning with ensemble cross-domain knowledge distillation.
\newblock In \emph{Proceedings of the AAAI Conference on Artificial Intelligence}, volume~36, pp.\  11891--11899, 2022.

\bibitem[Goodfellow et~al.(2014)Goodfellow, Pouget-Abadie, Mirza, Xu, Warde-Farley, Ozair, Courville, and Bengio]{goodfellow2014generative}
Goodfellow, I., Pouget-Abadie, J., Mirza, M., Xu, B., Warde-Farley, D., Ozair, S., Courville, A., and Bengio, Y.
\newblock Generative adversarial nets.
\newblock \emph{Advances in neural information processing systems}, 27, 2014.

\bibitem[Gretton et~al.(2012)Gretton, Borgwardt, Rasch, Sch{\"o}lkopf, and Smola]{gretton2012kernel}
Gretton, A., Borgwardt, K.~M., Rasch, M.~J., Sch{\"o}lkopf, B., and Smola, A.
\newblock A kernel two-sample test.
\newblock \emph{The Journal of Machine Learning Research}, 13\penalty0 (1):\penalty0 723--773, 2012.

\bibitem[Griffin et~al.(2007)Griffin, Holub, and Perona]{griffin2007caltech}
Griffin, G., Holub, A., and Perona, P.
\newblock Caltech-256 object category dataset.
\newblock 2007.

\bibitem[Hendrycks \& Dietterich(2019)Hendrycks and Dietterich]{hendrycks2019benchmarking}
Hendrycks, D. and Dietterich, T.
\newblock Benchmarking neural network robustness to common corruptions and perturbations.
\newblock \emph{arXiv preprint arXiv:1903.12261}, 2019.

\bibitem[Hinton et~al.(2015)Hinton, Vinyals, and Dean]{hinton2015distilling}
Hinton, G., Vinyals, O., and Dean, J.
\newblock Distilling the knowledge in a neural network.
\newblock \emph{arXiv preprint arXiv:1503.02531}, 2015.

\bibitem[Huang et~al.(2022)Huang, Ye, and Du]{huang2022learn}
Huang, W., Ye, M., and Du, B.
\newblock Learn from others and be yourself in heterogeneous federated learning.
\newblock In \emph{Proceedings of the IEEE/CVF Conference on Computer Vision and Pattern Recognition}, pp.\  10143--10153, 2022.

\bibitem[Huang et~al.(2021{\natexlab{a}})Huang, Chu, Zhou, Wang, Liu, Pei, and Zhang]{huang2021personalized}
Huang, Y., Chu, L., Zhou, Z., Wang, L., Liu, J., Pei, J., and Zhang, Y.
\newblock Personalized cross-silo federated learning on non-iid data.
\newblock In \emph{Proceedings of the AAAI conference on artificial intelligence}, volume~35, pp.\  7865--7873, 2021{\natexlab{a}}.

\bibitem[Huang et~al.(2021{\natexlab{b}})Huang, Gupta, Song, Li, and Arora]{huang2021evaluating}
Huang, Y., Gupta, S., Song, Z., Li, K., and Arora, S.
\newblock Evaluating gradient inversion attacks and defenses in federated learning.
\newblock \emph{Advances in Neural Information Processing Systems}, 34:\penalty0 7232--7241, 2021{\natexlab{b}}.

\bibitem[Hull(1994)]{hull1994database}
Hull, J.~J.
\newblock A database for handwritten text recognition research.
\newblock \emph{IEEE Transactions on pattern analysis and machine intelligence}, 16\penalty0 (5):\penalty0 550--554, 1994.

\bibitem[Jeong et~al.(2018)Jeong, Oh, Kim, Park, Bennis, and Kim]{jeong2018federated}
Jeong, E., Oh, S., Kim, H., Park, J., Bennis, M., and Kim, S.
\newblock Federated distillation and augmentation under non-iid private data.
\newblock \emph{NIPS Wksp. MLPCD}, 2018.

\bibitem[Jiang et~al.(2023)Jiang, Yang, Cheng, and Dou]{jiang2023iop}
Jiang, M., Yang, H., Cheng, C., and Dou, Q.
\newblock Iop-fl: Inside-outside personalization for federated medical image segmentation.
\newblock \emph{IEEE Transactions on Medical Imaging}, 2023.

\bibitem[Karimireddy et~al.(2020)Karimireddy, Kale, Mohri, Reddi, Stich, and Suresh]{karimireddy2020scaffold}
Karimireddy, S.~P., Kale, S., Mohri, M., Reddi, S., Stich, S., and Suresh, A.~T.
\newblock Scaffold: Stochastic controlled averaging for federated learning.
\newblock In \emph{International Conference on Machine Learning}, pp.\  5132--5143. PMLR, 2020.

\bibitem[Karras et~al.(2019)Karras, Laine, and Aila]{karras2019style}
Karras, T., Laine, S., and Aila, T.
\newblock A style-based generator architecture for generative adversarial networks.
\newblock In \emph{Proceedings of the IEEE/CVF conference on computer vision and pattern recognition}, pp.\  4401--4410, 2019.

\bibitem[Krizhevsky et~al.(2009)Krizhevsky, Hinton, et~al.]{krizhevsky2009learning}
Krizhevsky, A., Hinton, G., et~al.
\newblock Learning multiple layers of features from tiny images.
\newblock 2009.

\bibitem[LeCun et~al.(1998)LeCun, Bottou, Bengio, and Haffner]{lecun1998gradient}
LeCun, Y., Bottou, L., Bengio, Y., and Haffner, P.
\newblock Gradient-based learning applied to document recognition.
\newblock \emph{Proceedings of the IEEE}, 86\penalty0 (11):\penalty0 2278--2324, 1998.

\bibitem[Lee et~al.(2022)Lee, Chun, Jung, Yun, and Yoon]{lee2022dataset}
Lee, S., Chun, S., Jung, S., Yun, S., and Yoon, S.
\newblock Dataset condensation with contrastive signals.
\newblock In \emph{International Conference on Machine Learning}, pp.\  12352--12364. PMLR, 2022.

\bibitem[Li et~al.(2021{\natexlab{a}})Li, Li, and Varshney]{li2021decentralized}
Li, C., Li, G., and Varshney, P.~K.
\newblock Decentralized federated learning via mutual knowledge transfer.
\newblock \emph{IEEE Internet of Things Journal}, 9\penalty0 (2):\penalty0 1136--1147, 2021{\natexlab{a}}.

\bibitem[Li \& Wang(2019)Li and Wang]{li2019fedmd}
Li, D. and Wang, J.
\newblock Fedmd: Heterogenous federated learning via model distillation.
\newblock \emph{arXiv preprint arXiv:1910.03581}, 2019.

\bibitem[Li et~al.(2022)Li, Togo, Ogawa, and Haseyama]{li2022dataset}
Li, G., Togo, R., Ogawa, T., and Haseyama, M.
\newblock Dataset distillation for medical dataset sharing.
\newblock \emph{arXiv preprint arXiv:2209.14603}, 2022.

\bibitem[Li et~al.(2021{\natexlab{b}})Li, He, and Song]{li2021model}
Li, Q., He, B., and Song, D.
\newblock Model-contrastive federated learning.
\newblock In \emph{Proceedings of the IEEE/CVF Conference on Computer Vision and Pattern Recognition}, pp.\  10713--10722, 2021{\natexlab{b}}.

\bibitem[Li et~al.(2020{\natexlab{a}})Li, Sahu, Zaheer, Sanjabi, Talwalkar, and Smith]{li2020federated}
Li, T., Sahu, A.~K., Zaheer, M., Sanjabi, M., Talwalkar, A., and Smith, V.
\newblock Federated optimization in heterogeneous networks.
\newblock \emph{Proceedings of Machine learning and systems}, 2:\penalty0 429--450, 2020{\natexlab{a}}.

\bibitem[Li et~al.(2020{\natexlab{b}})Li, Huang, Yang, Wang, and Zhang]{li2019convergence}
Li, X., Huang, K., Yang, W., Wang, S., and Zhang, Z.
\newblock On the convergence of fedavg on non-iid data.
\newblock \emph{International Conference on Learning Representations}, 2020{\natexlab{b}}.

\bibitem[Lin et~al.(2020)Lin, Kong, Stich, and Jaggi]{lin2020ensemble}
Lin, T., Kong, L., Stich, S.~U., and Jaggi, M.
\newblock Ensemble distillation for robust model fusion in federated learning.
\newblock \emph{Advances in Neural Information Processing Systems}, 33:\penalty0 2351--2363, 2020.

\bibitem[Luo \& Ye(2022)Luo and Ye]{luo2022decentralized}
Luo, L. and Ye, H.
\newblock Decentralized stochastic variance reduced extragradient method.
\newblock \emph{arXiv preprint arXiv:2202.00509}, 2022.

\bibitem[Marfoq et~al.(2021)Marfoq, Neglia, Bellet, Kameni, and Vidal]{marfoq2021federated}
Marfoq, O., Neglia, G., Bellet, A., Kameni, L., and Vidal, R.
\newblock Federated multi-task learning under a mixture of distributions.
\newblock \emph{Advances in Neural Information Processing Systems}, 34:\penalty0 15434--15447, 2021.

\bibitem[Matsuda et~al.(2022)Matsuda, Sasaki, Xiao, and Onizuka]{matsuda2022fedme}
Matsuda, K., Sasaki, Y., Xiao, C., and Onizuka, M.
\newblock Fedme: Federated learning via model exchange.
\newblock In \emph{Proceedings of the 2022 SIAM international conference on data mining (SDM)}, pp.\  459--467. SIAM, 2022.

\bibitem[McMahan et~al.(2017)McMahan, Moore, Ramage, Hampson, and y~Arcas]{mcmahan2017communication}
McMahan, B., Moore, E., Ramage, D., Hampson, S., and y~Arcas, B.~A.
\newblock Communication-efficient learning of deep networks from decentralized data.
\newblock In \emph{Artificial intelligence and statistics}, pp.\  1273--1282. PMLR, 2017.

\bibitem[Netzer et~al.(2011)Netzer, Wang, Coates, Bissacco, Wu, and Ng]{netzer2011reading}
Netzer, Y., Wang, T., Coates, A., Bissacco, A., Wu, B., and Ng, A.~Y.
\newblock Reading digits in natural images with unsupervised feature learning.
\newblock 2011.

\bibitem[Pappas et~al.(2021)Pappas, Chatzopoulos, Lalis, and Vavalis]{pappas2021ipls}
Pappas, C., Chatzopoulos, D., Lalis, S., and Vavalis, M.
\newblock Ipls: A framework for decentralized federated learning.
\newblock In \emph{2021 IFIP Networking Conference (IFIP Networking)}, pp.\  1--6. IEEE, 2021.

\bibitem[Rostami(2021)]{rostami2021lifelong}
Rostami, M.
\newblock Lifelong domain adaptation via consolidated internal distribution.
\newblock \emph{Advances in neural information processing systems}, 34:\penalty0 11172--11183, 2021.

\bibitem[Roy et~al.(2019)Roy, Siddiqui, P{\"o}lsterl, Navab, and Wachinger]{roy2019braintorrent}
Roy, A.~G., Siddiqui, S., P{\"o}lsterl, S., Navab, N., and Wachinger, C.
\newblock Braintorrent: A peer-to-peer environment for decentralized federated learning.
\newblock \emph{arXiv preprint arXiv:1905.06731}, 2019.

\bibitem[Saenko et~al.(2010)Saenko, Kulis, Fritz, and Darrell]{saenko2010adapting}
Saenko, K., Kulis, B., Fritz, M., and Darrell, T.
\newblock Adapting visual category models to new domains.
\newblock In \emph{Computer Vision--ECCV 2010: 11th European Conference on Computer Vision, Heraklion, Crete, Greece, September 5-11, 2010, Proceedings, Part IV 11}, pp.\  213--226. Springer, 2010.

\bibitem[Sheller et~al.(2019)Sheller, Reina, Edwards, Martin, and Bakas]{sheller2019multi}
Sheller, M.~J., Reina, G.~A., Edwards, B., Martin, J., and Bakas, S.
\newblock Multi-institutional deep learning modeling without sharing patient data: A feasibility study on brain tumor segmentation.
\newblock In \emph{Brainlesion: Glioma, Multiple Sclerosis, Stroke and Traumatic Brain Injuries: 4th International Workshop, BrainLes 2018, Held in Conjunction with MICCAI 2018, Granada, Spain, September 16, 2018, Revised Selected Papers, Part I 4}, pp.\  92--104. Springer, 2019.

\bibitem[Sheller et~al.(2020)Sheller, Edwards, Reina, Martin, Pati, Kotrotsou, Milchenko, Xu, Marcus, Colen, et~al.]{sheller2020federated}
Sheller, M.~J., Edwards, B., Reina, G.~A., Martin, J., Pati, S., Kotrotsou, A., Milchenko, M., Xu, W., Marcus, D., Colen, R.~R., et~al.
\newblock Federated learning in medicine: facilitating multi-institutional collaborations without sharing patient data.
\newblock \emph{Scientific reports}, 10\penalty0 (1):\penalty0 12598, 2020.

\bibitem[Shen et~al.(2023)Shen, Zhang, Jia, Zhang, Lv, Kuang, Wu, and Wu]{shen2023federated}
Shen, T., Zhang, J., Jia, X., Zhang, F., Lv, Z., Kuang, K., Wu, C., and Wu, F.
\newblock Federated mutual learning: a collaborative machine learning method for heterogeneous data, models, and objectives.
\newblock \emph{Frontiers of Information Technology \& Electronic Engineering}, 24\penalty0 (10):\penalty0 1390--1402, 2023.

\bibitem[Shokri et~al.(2017)Shokri, Stronati, Song, and Shmatikov]{shokri2017membership}
Shokri, R., Stronati, M., Song, C., and Shmatikov, V.
\newblock Membership inference attacks against machine learning models.
\newblock In \emph{2017 IEEE symposium on security and privacy (SP)}, pp.\  3--18. IEEE, 2017.

\bibitem[Song et~al.(2023)Song, Liu, Chen, Festag, Trinitis, Schulz, and Knoll]{song2023federated}
Song, R., Liu, D., Chen, D.~Z., Festag, A., Trinitis, C., Schulz, M., and Knoll, A.
\newblock Federated learning via decentralized dataset distillation in resource-constrained edge environments.
\newblock In \emph{2023 International Joint Conference on Neural Networks (IJCNN)}, pp.\  1--10. IEEE, 2023.

\bibitem[Tan et~al.(2022)Tan, Long, Liu, Zhou, Lu, Jiang, and Zhang]{tan2022fedproto}
Tan, Y., Long, G., Liu, L., Zhou, T., Lu, Q., Jiang, J., and Zhang, C.
\newblock Fedproto: Federated prototype learning across heterogeneous clients.
\newblock In \emph{Proceedings of the AAAI Conference on Artificial Intelligence}, volume~36, pp.\  8432--8440, 2022.

\bibitem[Tang et~al.(2022)Tang, Zhang, Shi, He, Han, and Chu]{tang2022virtual}
Tang, Z., Zhang, Y., Shi, S., He, X., Han, B., and Chu, X.
\newblock Virtual homogeneity learning: Defending against data heterogeneity in federated learning.
\newblock \emph{arXiv preprint arXiv:2206.02465}, 2022.

\bibitem[Wang et~al.(2023)Wang, Chen, Kerkouche, and Fritz]{wang2023fed}
Wang, H.-P., Chen, D., Kerkouche, R., and Fritz, M.
\newblock Fed-gloss-dp: Federated, global learning using synthetic sets with record level differential privacy.
\newblock \emph{arXiv preprint arXiv:2302.01068}, 2023.

\bibitem[Wang et~al.(2020)Wang, Liu, Liang, Joshi, and Poor]{wang2020tackling}
Wang, J., Liu, Q., Liang, H., Joshi, G., and Poor, H.~V.
\newblock Tackling the objective inconsistency problem in heterogeneous federated optimization.
\newblock \emph{Advances in neural information processing systems}, 33:\penalty0 7611--7623, 2020.

\bibitem[Wang et~al.(2022)Wang, Zhao, Peng, Zhu, Yang, Wang, Huang, Bilen, Wang, and You]{wang2022cafe}
Wang, K., Zhao, B., Peng, X., Zhu, Z., Yang, S., Wang, S., Huang, G., Bilen, H., Wang, X., and You, Y.
\newblock Cafe: Learning to condense dataset by aligning features.
\newblock In \emph{Proceedings of the IEEE/CVF Conference on Computer Vision and Pattern Recognition}, pp.\  12196--12205, 2022.

\bibitem[Wu et~al.(2022)Wu, Wu, Lyu, Huang, and Xie]{wu2022communication}
Wu, C., Wu, F., Lyu, L., Huang, Y., and Xie, X.
\newblock Communication-efficient federated learning via knowledge distillation.
\newblock \emph{Nature communications}, 13\penalty0 (1):\penalty0 2032, 2022.

\bibitem[Xiao et~al.(2017)Xiao, Rasul, and Vollgraf]{xiao2017fashion}
Xiao, H., Rasul, K., and Vollgraf, R.
\newblock Fashion-mnist: a novel image dataset for benchmarking machine learning algorithms.
\newblock \emph{arXiv preprint arXiv:1708.07747}, 2017.

\bibitem[Xiong et~al.(2023)Xiong, Wang, Cheng, Yu, and Hsieh]{xiong2023feddm}
Xiong, Y., Wang, R., Cheng, M., Yu, F., and Hsieh, C.-J.
\newblock Feddm: Iterative distribution matching for communication-efficient federated learning.
\newblock In \emph{Proceedings of the IEEE/CVF Conference on Computer Vision and Pattern Recognition}, pp.\  16323--16332, 2023.

\bibitem[Yang et~al.(2021)Yang, Tian, and Zhang]{yang2021regularized}
Yang, R., Tian, J., and Zhang, Y.
\newblock Regularized mutual learning for personalized federated learning.
\newblock In \emph{Asian Conference on Machine Learning}, pp.\  1521--1536. PMLR, 2021.

\bibitem[Ye et~al.(2020)Ye, Zhou, Luo, and Zhang]{ye2020decentralized}
Ye, H., Zhou, Z., Luo, L., and Zhang, T.
\newblock Decentralized accelerated proximal gradient descent.
\newblock \emph{Advances in Neural Information Processing Systems}, 33:\penalty0 18308--18317, 2020.

\bibitem[Ye et~al.(2022)Ye, Ni, Xu, Wang, Chen, and Eldar]{ye2022fedfm}
Ye, R., Ni, Z., Xu, C., Wang, J., Chen, S., and Eldar, Y.~C.
\newblock Fedfm: Anchor-based feature matching for data heterogeneity in federated learning.
\newblock \emph{arXiv preprint arXiv:2210.07615}, 2022.

\bibitem[Yuan et~al.(2023{\natexlab{a}})Yuan, Ma, Su, and Wang]{yuan2023peer}
Yuan, L., Ma, Y., Su, L., and Wang, Z.
\newblock Peer-to-peer federated continual learning for naturalistic driving action recognition.
\newblock In \emph{Proceedings of the IEEE/CVF Conference on Computer Vision and Pattern Recognition}, pp.\  5249--5258, 2023{\natexlab{a}}.

\bibitem[Yuan et~al.(2023{\natexlab{b}})Yuan, Sun, Yu, and Wang]{yuan2023decentralized}
Yuan, L., Sun, L., Yu, P.~S., and Wang, Z.
\newblock Decentralized federated learning: A survey and perspective.
\newblock \emph{arXiv preprint arXiv:2306.01603}, 2023{\natexlab{b}}.

\bibitem[Zhang et~al.(2022{\natexlab{a}})Zhang, Chen, Li, Lyu, Wu, Ding, Shen, and Wu]{zhang2022dense}
Zhang, J., Chen, C., Li, B., Lyu, L., Wu, S., Ding, S., Shen, C., and Wu, C.
\newblock Dense: Data-free one-shot federated learning.
\newblock \emph{Advances in Neural Information Processing Systems}, 35:\penalty0 21414--21428, 2022{\natexlab{a}}.

\bibitem[Zhang et~al.(2022{\natexlab{b}})Zhang, Shen, Ding, Tao, and Duan]{zhang2022fine}
Zhang, L., Shen, L., Ding, L., Tao, D., and Duan, L.-Y.
\newblock Fine-tuning global model via data-free knowledge distillation for non-iid federated learning.
\newblock In \emph{Proceedings of the IEEE/CVF conference on computer vision and pattern recognition}, pp.\  10174--10183, 2022{\natexlab{b}}.

\bibitem[Zhang et~al.(2018)Zhang, Xiang, Hospedales, and Lu]{zhang2018deep}
Zhang, Y., Xiang, T., Hospedales, T.~M., and Lu, H.
\newblock Deep mutual learning.
\newblock In \emph{Proceedings of the IEEE conference on computer vision and pattern recognition}, pp.\  4320--4328, 2018.

\bibitem[Zhao \& Bilen(2021)Zhao and Bilen]{zhao2021dataset}
Zhao, B. and Bilen, H.
\newblock Dataset condensation with differentiable siamese augmentation.
\newblock In \emph{International Conference on Machine Learning}, pp.\  12674--12685. PMLR, 2021.

\bibitem[Zhao \& Bilen(2022)Zhao and Bilen]{zhao2022synthesizing}
Zhao, B. and Bilen, H.
\newblock Synthesizing informative training samples with gan.
\newblock \emph{arXiv preprint arXiv:2204.07513}, 2022.

\bibitem[Zhao \& Bilen(2023)Zhao and Bilen]{zhao2023dataset}
Zhao, B. and Bilen, H.
\newblock Dataset condensation with distribution matching.
\newblock In \emph{Proceedings of the IEEE/CVF Winter Conference on Applications of Computer Vision}, pp.\  6514--6523, 2023.

\bibitem[Zhao et~al.(2020)Zhao, Mopuri, and Bilen]{zhao2020dataset}
Zhao, B., Mopuri, K.~R., and Bilen, H.
\newblock Dataset condensation with gradient matching.
\newblock In \emph{International Conference on Learning Representations}, 2020.

\bibitem[Zhou et~al.(2022)Zhou, Zhang, and Tsang]{zhou2022fedfa}
Zhou, T., Zhang, J., and Tsang, D.
\newblock Fedfa: Federated learning with feature anchors to align feature and classifier for heterogeneous data.
\newblock \emph{arXiv preprint arXiv:2211.09299}, 2022.

\bibitem[Zhu et~al.(2021)Zhu, Hong, and Zhou]{zhu2021data}
Zhu, Z., Hong, J., and Zhou, J.
\newblock Data-free knowledge distillation for heterogeneous federated learning.
\newblock In \emph{International conference on machine learning}, pp.\  12878--12889. PMLR, 2021.

\end{thebibliography}
\bibliographystyle{icml2024}

\newpage
\appendix
\onecolumn

\textbf{Road Map of Appendix}
Our appendix is organized into five sections. The theoretical analysis and proof is in Appendix~\ref{app:proof}. Appendix~\ref{app:MIA} shows the results for Membership Inference Attack (MIA) on \ours{} trained models using \texttt{DIGITS} datasets. Appendix~\ref{app:DP} discusses how we inject DP mechanism in our data synthesis process, and shows that using DP synthetic anchor data for \ours{} can still yeilds comparable performance. Appendix~\ref{app:synthetic_images} introduce the selected datasets and how we synthesize anchor data in detail. Appendix~\ref{app:model_arch} describes the model architectures (ConvNet and AlexNet) we use in our experiments. Finally, Appendix~\ref{app:related} provides a detailed literature review about the related works. Our code and model checkpoints are available along with the supplementary materials.

\section{Theoretical Analysis and Proofs}\label{app:proof}

\subsection{Notation}
\begin{table}[htbp]\caption{Notations used}
\begin{center}
\begin{tabular}{r c p{10cm} }
\toprule
$\bm{D_i} = (x_i, y_i)_{i = 1}^{i = m}$ & $\triangleq$ & Local Dataset of client $C_i$\\ 
$P_i(x, y)$ & $\triangleq$ & The local joint  distribution of client $C_i$\\
$L_{CE}(a, b)$ & $\triangleq$ & Cross entropy Loss between distributions $a$ and $b$\\
$L_{KL}(a, b)$ & $\triangleq$ & KL divergence loss between distributions($a, b$)\\  
$\mathcal{M}_i$ & $\triangleq$ & Model Space of Client $C_i$
\\
$\mathcal{P}_i$ & $\triangleq$ & Space of Classifier heads for Client $C_i$
\\
$\Psi_i$ & $\triangleq$ & Space of encoder heads for Client $C_i$
\\
$M_i = \rho_i \circ \psi_i$ & $\triangleq$ & Model for client $C_i$ with encoder and decoder heads sampled from $\mathcal{P}_i$ and $\Psi_i$
\\
$\bm{\alpha} = [ \alpha, \alpha^{Syn}, \alpha^{Syn}_{KD}]$ & $\triangleq$ & Component weights for the losses, as defined in Eq. ~\ref{Eq:Hull}
\\
\multicolumn{3}{c}{}\\
\multicolumn{3}{c}{\underline{Notations from \cite{ben2010theory}}}\\
\multicolumn{3}{c}{}\\
$(P, f^P)$ & $\triangleq$ & $P(x)$ is the source data distribution and $f^P$ is the optimal labelling function \\ 
$\epsilon_{P}(M)$ & $\triangleq$ & $\Pr_{x \sim P(x)} (M(x) \neq f_P(x)) $\\
$d_{\mathcal{H}\Delta\mathcal{H}}(P_s, P_t)$ & $\triangleq$ &$ 2 \underset{h,h' \in \mathcal{H}}{\sup} | \Pr_{x \sim P_s} (h(x) \neq h(x'))| - \Pr_{x \sim P_t} (h(x) \neq h(x'))|  $\\
$\lambda(P)$ & $\triangleq$ & Least error of a jointly trained model =$\min_{M \in \mathcal{M}} \epsilon_P(M) + \epsilon_{P^T}(M)$\\
$\mathbf{C}(P_i, P^T)$ & $\triangleq$ & A distance term appearing in \cite{ben2010theory} $ = \frac{1}{2} (d_{\mathcal{M}\Delta \mathcal{M}} (P_i, P^T) + \lambda(P_i)$
\\
\bottomrule
\end{tabular}
\end{center}
\label{tab:TableOfNotationForMyResearch}
\end{table}

\subsection{Proof for Theorem \ref{Thm:general_bound}}

\begin{proof}
The training data at $i$th client are from as three distributions: 1) the local source data; 2) the global virtual data; 3) the extended KD data. The data from the first two groups are used for the cross entropy loss and the distribution divergence, while the third is used for Knowledge distillation.

Without loss of generality, at $i$th client, we set the weight for $P_i$, $P^{Syn}$ and ${P}_{KD}^{Syn}$ as $\alpha$, $\alpha^{Syn}$ and ${\alpha}_{KD}^{Syn}$, respectively. For notation simplicity, we assume $\alpha+\alpha^{Syn}+{\alpha}_{KD}^{Syn}=1$. 
Then the training source data at $i$th client is ${P}_i^{S} = 
\alpha {P}_i + \alpha^{Syn}{P}^{Syn} + {\alpha}_{KD}^{Syn}{P}_{KD}^{Syn}$.

From Theorem 2 in \cite{ben2010theory}, it holds that 
\begin{align}\label{eq:loss_on_Pi}
    \epsilon_{{P}^T}(M_i) \le \epsilon_{{P}_i}(M_i) + \mathbf{C}(P_i, P_T) 
\end{align}
where $\mathbf{C}(P_i, P_T) = \frac{1}{2} d_{\mathcal{M}_i\Delta\mathcal{M}_i}({P}_i ,{P}^{T}) + \lambda({P}_i)$ and $\lambda(P_i) = \min_{M\in\mathcal{M}_i}\epsilon_{{P}_i}(M)+\epsilon_{{P}^T}(M)$ is a constant. These terms are small unless the data heterogeneity is severe \cite{ben2010theory}.
Then with~(\ref{eq:loss_on_Pi}) and Lemma~\ref{Lem:bound_for_virtual}, 
we have three inequalities, which we will add after multiplying each one of them with their corresponding component weight $\bm{\alpha}$.

Furthermore, note that the support of $P^{Syn}$ and $P^{Syn}_{KD}$ are the same. (\ref{Def: Extended KD}). Therefore, their distances from the support of the global distribution will also be the same, i.e.

$$
    d_{\mathcal{P}_i \Delta \mathcal{P}_i} (\psi \circ P^{Syn}, \psi \circ P^{T}) = d_{\mathcal{P}_i \Delta \mathcal{P}_i} (\psi \circ P^{Syn}_{KD}, \psi \circ P^{T})
$$

Continuing after adding all three inequalities and using Claim ~\ref{Claim: equivalent} to introduce $\epsilon_{P_i^S}(M)$
\begin{align}
    \epsilon_{{P}^T}(M_i) \le &  \epsilon_{{P}_i^S}(M_i) + \mathbf{C}(P_i, P^T)  +  \frac{\alpha^{Syn}}{2} d_{\mathcal{P}\Delta \mathcal{P}} (\psi \circ P^{Syn}, \psi \circ P^T) \notag\\
    &
    +  \epsilon_{{P}^T}(f^{Syn}) ) + \frac{\alpha^{Syn}_{KD}}{2} d_{\mathcal{P}\Delta \mathcal{P}} (\psi \circ P^{Syn}, \psi \circ P^T) +  \epsilon_{{P}_{KD}^T}(f^{Syn}) ) \notag\\
    \le &  \epsilon_{{P}_i^S}(M_i) + \alpha \mathbf{C}(P_i, P^T) 
    +  \alpha^{Syn}\epsilon_{{P}^T}(f^{Syn})  + \alpha^{Syn}_{KD}  \epsilon_{{P}_{KD}^T}(f^{Syn}) \notag\\
    & + \frac{(1 - \alpha)}{2} d_{\mathcal{P}\Delta \mathcal{P}} (\psi \circ P^{Syn}, \psi \circ P^T)
\end{align}

With the last condition coming from $\alpha + \alpha^{Syn} + \alpha^{Syn}_{KD} = 1$ 
\begin{align}
    \epsilon_{{P}^T}(M_i) \le  \epsilon_{{P}_i^S}(M_i) + \alpha \mathbf{C}(P_i, P^T) 
    +  \alpha^{Syn}\epsilon_{{P}^T}(f^{Syn})  + \alpha^{Syn}_{KD}  \epsilon_{{P}^T}(f_{KD}^{Syn}) + \frac{(1 - \alpha)}{2} d_{\mathcal{P}\Delta \mathcal{P}} (\psi \circ P^{Syn}, \psi \circ P^T)
\end{align}
\end{proof}

\subsection{Interpretation for Theorem~\ref{Thm:general_bound}}\label{Thm:interpretation}
 From Eq.~(\ref{eq:general_bound}), it can be seen that the generalization bound for $M_i$ consists of five terms.
\begin{itemize}
    \item The first term $\epsilon_{{P}_i^S}(M_i)$ is the error bound with respect to the training source data distribution. With Claim~\ref{Claim: equivalent} in appendix, minimizing this term is equivalent to optimizing the loss $\alpha \mathbb{E}_{(\x,y)\sim P_i}\mathcal{L}_{\text{CE}} + \alpha^{Syn} \mathbb{E}_{(\x,y)\sim P^{Syn}}\mathcal{L}_{\text{CE}} + \alpha^{Syn}_{\text{KD}}\mathbb{E}_{(\x,y)\sim P^{Syn}}\mathcal{L}_{\text{KD}}$. Since this is the form of our loss function  in Eq.~\ref{eq:glob}, we expect this term to be minimized 

    
    \item The second term is inherited from the original generalization bound in \cite{ben2010theory} with the local training data. For our case, it can be controlled by the component weight $\alpha$. If we rely less on the local data (i.e. smaller $\alpha$), then these terms will be vanishing. Moreover even if we rely more on local data, it is essentially a distance measure between the local client distribution $P_i$ and the global data distribution $P^T$. Since the global data distribution is an average of the closely related local client distributions, we expect this term to be small \cite{tang2022virtual}, \cite{ben2010theory}, \cite{albuquerque2019generalizing}.

    \item  The third term  measures the discrepancy between real labeling and the synthetic data labeling mechanisms. This discrepancy will be low because of our synthetic data generation process. Note that the data distillation's objective is to achieve $\mathbb{E}_{\x \sim P^T} [l(M^{T}(\x),y)] \simeq \mathbb{E}_{\x \sim P^{Syn}} [l(M^{Syn}(\x),y)]$ (Eq. 1 in \cite{zhao2023dataset}). If we change the $M$ to a well-trained deep NN, then it's easy to see the synthetic data labelling $f^{syn}$ will be similar to the real labelling $f_T$. Here we leverage the distribution matching that uses MMD loss to minimize the embedding differences between the synthetic data and real data in the same class~\cite{zhao2023dataset} as a proxy way to achieve that.
    
    \item The fourth term originates from the knowledge distillation loss in equation \ref{eq:kd}. Here, we use the consensus knowledge from neighbour models to improve the local model. The labelling function of the extended KD data $f^{syn}_{KD}$, changes as training continues and the neighbour models learn to generalize well. Towards the end of training, predictions from the consensus knowledge should match the predictions of the true labeling function, therefore, $f^{syn}_{KD}$ will be close to $f_T$.

    \item The fifth term is a distribution divergence between the encoded distributions of $P^{Syn}$ and $P^T$. This is minimized by the domain invariant regularizer in Eq. \ref{eq:lsgd}, which acts as an anchor to pull all the encoded distributions together.

\end{itemize}

\textbf{Remark:} In order to get a tight generalization guarantee, we only need one of the fourth or fifth terms to be small. Since, if either any one of them is small, we can adjust the component weights $\alpha$ (practically $\lambda_{REG}$ and $\lambda_{KD}$) to get a better generalization guarantee.

\subsection{Proof for Proposition~\ref{Prop:tighter_bound}}
\begin{proof}
Without loss of generality, let's start with 
\begin{align}
     \sup_{M\in\mathcal{M}_i}|\epsilon_{P^{Syn}}(M) - \epsilon_{{P}^T}(M)| +  & \epsilon_{{P}^T}(f^{Syn}) \le  \notag\\
     & \inf_{M\in\mathcal{M}_i}(\epsilon_{P_i}(M) - \epsilon_{{P}^T}(M)) + \underbrace{\frac{1}{2} d_{\mathcal{M}_i\Delta\mathcal{M}_i}({P}_i ,{P}^{T}) + \lambda({P}_i).}_{\mathbf{C}(P_i, P^T)} 
\end{align}
Then it holds that for any $M\in\mathcal{M}_i$,
\begin{align}\label{eq:relation_with_PG}
     \epsilon_{P^{Syn}}(M) - \epsilon_{{P}^T}(M) +  \epsilon_{{P}^T}(f^{Syn}) &\le \epsilon_{P_i}(M) - \epsilon_{{P}^T}(M) +   \mathbf{C}(P_i, P^T)\notag\\
    \Rightarrow~~  \epsilon_{P^{Syn}}(M) +  \epsilon_{{P}^T}(f^{Syn}) & \le \epsilon_{{P}_i}(M)+ \mathbf{C}(P_i, P^T)
\end{align}
Note that the right side of~(\ref{eq:relation_with_PG}) is the original bound in Theorem 2 in \cite{ben2010theory}. 
Similarly, we can achieve 
\begin{align}\label{eq:relation_with_PG_KL}
\epsilon_{P_{KD}^{Syn}}(M) +  \epsilon_{{P}^T}(f_{KD}^{Syn}) & \le \epsilon_{{P}_i}(M)+  \frac{1}{2} d_{\mathcal{M}_i\Delta\mathcal{M}_i}({P}_i ,{P}^{T}) + \lambda({P}_i) 
\end{align}
Combining (\ref{eq:relation_with_PG}-\ref{eq:relation_with_PG_KL}) together with the component weights $\bm{\alpha}$ and setting $\alpha \rightarrow 0$,

\begin{align}\label{eq:final_bound_relation}
\alpha^{Syn}_{KD}\epsilon_{P_{KD}^{Syn}}(M) +  
\alpha^{Syn}\epsilon_{P^{Syn}}(M) + 
\epsilon_{{P}^T}(f_{KD}^{Syn}) & \le \epsilon_{{P}_i}(M)+  \frac{1}{2} d_{\mathcal{M}_i\Delta\mathcal{M}_i}({P}_i ,{P}^{T}) + \lambda({P}_i) 
\end{align}

Therefore, we conclude that our global generalization bound in Theorem \ref{Thm:general_bound} (which is the LHS of Eq.~\ref{eq:final_bound_relation}) is tighter than the original bound in \cite{ben2010theory} (RHS), when the condition of Proposition \ref{Prop:tighter_bound} holds. 

\end{proof}

\subsection{Some useful lemmas and claims}

\begin{lem}\label{Lem:bound_for_virtual}
Denote the model as $M = \rho \circ \psi \in \mathcal{M}$. The global generalization bound holds as 
\begin{align}
     \epsilon_{{P}^T}(M)  \le  \epsilon_{{P}}(M) + \frac{1}{2} d_{\mathcal{P}\Delta \mathcal{P}} (\psi \circ P, \psi \circ P^T) +  \epsilon_{{P}^T}(f),
\end{align}
where $(P, f)$ could be either $(P^{Syn}, f^{Syn})$ or $(P^{Syn}_{KD},f^{Syn}_{KD})$ pair.
\end{lem}
\begin{proof}
For any model $M = \rho \circ \psi \in \mathcal{M}$, 
we have the following bound for the global virtual data distribution:
\begin{align}\label{eq:loss_on_Pg}
    \epsilon_{{P}^T}(M) - \epsilon_{{P}^{Syn}}(M) &\stackrel{(a)}{=} \epsilon_{{P}^T}(M, f^T) -  \epsilon_{{P}^{Syn}}(M, f^{Syn}) \notag\\
    &\stackrel{(b)}{\le} |\epsilon_{{P}^T}(M, f^{Syn}) + \epsilon_{{P}^T}(f^{Syn}, f^T)-  \epsilon_{{P}^{Syn}}(M, f^{Syn})| \notag\\
    &\le |\epsilon_{{P}^T}(M, f^{Syn}) -  \epsilon_{{P}^{Syn}}(M, f^{Syn}) | +  \epsilon_{{P}^T}(f^{Syn})  \notag\\
    &= |\epsilon_{{P}^T}(\rho \circ \psi, f^{Syn}) -  \epsilon_{{P}^{Syn}}(\rho \circ \psi, f^{Syn}) | +  \epsilon_{{P}^T}(f^{Syn}) \notag\\
    &= |\epsilon_{\psi \circ {P}^T}(\rho , f^{Syn}\circ\psi^{-1}) -  \epsilon_{\psi \circ {P}^{Syn}}(\rho , f^{Syn}\circ\psi^{-1}) | +  \epsilon_{{P}^T}(f^{Syn}) \notag\\
    &\le \sup_{\rho,\rho^\prime} |\epsilon_{\psi \circ {P}^T}(\rho , \rho^\prime) -  \epsilon_{\psi \circ {P}^{Syn}}(\rho , \rho^\prime) | +  \epsilon_{{P}^T}(f^{Syn}) 
    \notag \\
    &\le \frac{1}{2} d_{\mathcal{P}\Delta \mathcal{P}} (\psi \circ P^{Syn}, \psi \circ P^T) + \epsilon_{P^T}(f^{Syn})
\end{align}
where (a) is by definitions and (b) relies on the triangle inequality for classification error \cite{ben2006analysis,crammer2008learning}.
Thus, we have that 
\begin{align}
     \epsilon_{{P}^T}(M)  \le  \epsilon_{{P}^{Syn}}(M) + \frac{1}{2} d_{\mathcal{P}\Delta \mathcal{P}} (\psi \circ P^{Syn}, \psi \circ P^T)  +  \epsilon_{{P}^T}(f^{Syn}).
\end{align}
Similarly, as the the extended KD dataset shares the same feature distribution with the global virtual dataset, thus the above bound also holds for $f^{Syn}_{KD}$.
\end{proof}

\begin{lem}[Appendix A \cite{feng2021kd3a}]\label{Lem:KL_to_bound}
    For the extended source domain $(\x^{Syn}, \hat{y}^{Syn})\sim \hat{P}^{Syn}$, training the related model with the knowledge distillation loss $L_{KD}=D_{KD}(\hat{y}^{Syn}\|h(\x))$ equals to optimizing the task risk $\epsilon_{\hat{P}^{Syn}}=\mathbb{P}_{(\x^{Syn}, \hat{y}^{Syn})\sim \hat{P}^{Syn}}[h(\x)\neq \arg\max \hat{y}^{Syn}]$.
\end{lem}

\begin{claim}\label{Claim: equivalent}
With the training source data at $i$th client as ${P}_i^{S}$ with the component weight $\boldsymbol{\alpha} = [\alpha, \alpha^{Syn},{\alpha}_{KD}^{Syn}]^\top$ on the local data, virtual data and extended KD data, $\epsilon_{{P}_i^S}(h)$ is minimized by optimizing the loss:
\begin{align}
    \min_{M\in\mathcal{M}} \alpha\mathbb{E}_{(\x,y)\sim {P}_i}L_{CE}(y,M(\x))\!+ \!
    \alpha^{Syn}\mathbb{E}_{(\x,y)\sim {P}^{Syn}}L_{CE}(y,M(\x))\!+ \!
    {\alpha}_{KD}^{Syn}\mathbb{E}_{(\x,y)\sim {P}_{KD}^{Syn}}L_{KL}(y\|M(\x))
\end{align}
\end{claim}

\begin{proof}
Note that 
\begin{align}
    &\min_{M\in\mathcal{M}} \mathbb{E}_{(\x,y)\sim {P}_i^{(S)}}L_{KL}(y\|M(\x)) \notag\\
    &\propto \!\!\min_{M\in\mathcal{M}} \alpha\mathbb{E}_{(\x,y)\sim {P}_i}L_{KL}(y\|M(\x))\!+ \!
    \alpha^{Syn}\mathbb{E}_{(\x,y)\sim {P}^{Syn}}L_{KL}(y\|M(\x)) \!+ \!
    {\alpha}_{KD}^{Syn}\mathbb{E}_{(\x,y)\sim {P}_{KD}^{Syn}}L_{KL}(y\|M(\x)) \notag\\
     &\stackrel{(a)}{\propto} \!\!\min_{M\in\mathcal{M}} \alpha\mathbb{E}_{(\x,y)\sim {P}_i}L_{CE}(y,M(\x))\!+ \!
    \alpha^{Syn}\mathbb{E}_{(\x,y)\sim {P}^{Syn}}L_{CE}(y,M(\x))\!+ \!
    {\alpha}_{KD}^{Syn}\mathbb{E}_{(\x,y)\sim {P}_{KD}^{Syn}}L_{KL}(y\|M(\x)) \notag
\end{align}
where (a) is because $L_{{KL}}(y\|h(\x)) = L_{CE}(y,h(\x)) - H(y)$, where $H(y) = - y\log(y)$ is a constant depending on data distribution.  With Lemma~\ref{Lem:KL_to_bound} and Pinsker’s inequality, it is easy to show that $\epsilon_{{P}_i^S}(h)$ is minimized by optimizing the above loss.
\end{proof}

\section{Privacy Discussion for \ours{}}
\label{app:privacy}

Sharing image-level information among clients may raise privacy concerns. However, we claim that decentralized FL with both data and model heterogeneities is an extremely challenging setting, where existing solutions either require sharing real public data~\cite{lin2020ensemble,huang2022learn} or synthetic data generated from real data with GAN-based generator~\cite{zhang2022dense,zhang2022fine}. Instead, we propose to use distribution matching to distill data, a simple and less data-greedy strategy, for data synthesis. Research has shown that using distilled data can defend against privacy attacks~\cite{dong2022privacy} such as membership inference attacks (MIA)~\cite{shokri2017membership} and gradient inversion attacks~\cite{huang2021evaluating}. We show the \ours{}'s defense against MIA~\cite{carlini2022membership} in Appendix~\ref{app:MIA}. In addition, recent papers have successfully applied differential privacy (DP)~\cite{abadi2016deep} mechanism into data distillation~\cite{xiong2023feddm,wang2023fed} to ensure privacy. We also discuss the feasibility of adding DP into data distillation process following~\cite{xiong2023feddm} and show that \ours{} is still effective using the DP synthetic anchor data in Appendix~\ref{app:DP}. 
We would like to put more emphasis on our proposed methodology and theoretical analysis in the main text, as sharing synthetic data commonly exists in the related work mentioned above and we fairly align with their settings in our comparisons. Thus, we consider the potential privacy risk of FL with \ours{} is beyond the main scope of our study and leave it in our Appendix.

\subsection{Membership Inference Attack}\label{app:MIA}
We show what under the basic setting of \ours{} (\ie{} not applying Differential Privacy when generating local synthetic data), we can better protect the membership information of local real data than local training or FedAvg~\cite{mcmahan2017communication} on local real data only when facing Membership Inference Attack (MIA) on trained local models. Although we share the logits during communication, it's important to note that these logits are from synthetic anchor data and not real data that needs protection. Therefore, we cannot use MIA methods that rely on logits. Instead, we  perform a strong MIA attack recently proposed and evaluate it following the approach in \cite{carlini2022membership}.

\begin{figure}[ht]
\centering
    \subfloat[SVHN]{\includegraphics[width=0.33\linewidth]{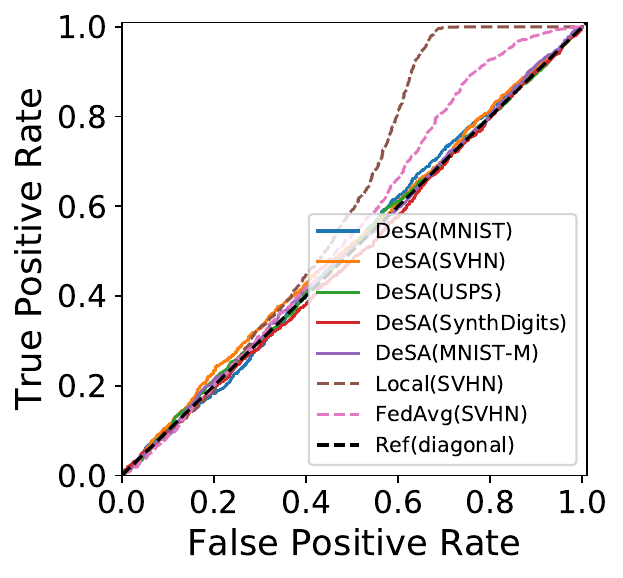}}
    \subfloat[SynthDigits]{\includegraphics[width=0.33\linewidth]{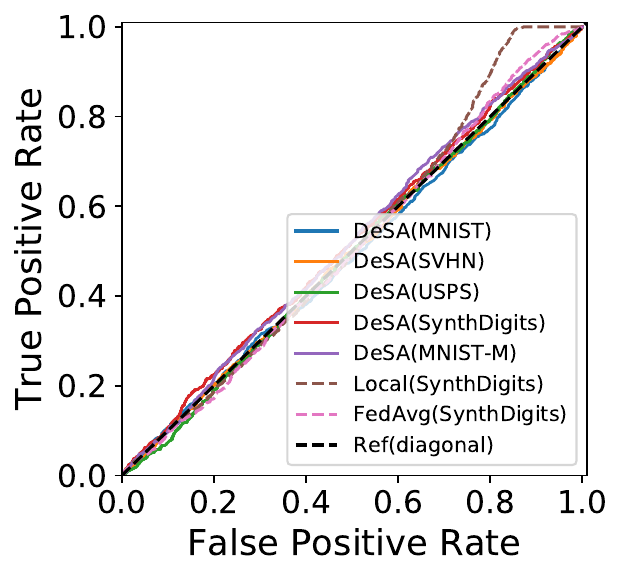}}
    \subfloat[MNIST-M]{\includegraphics[width=0.33\linewidth]{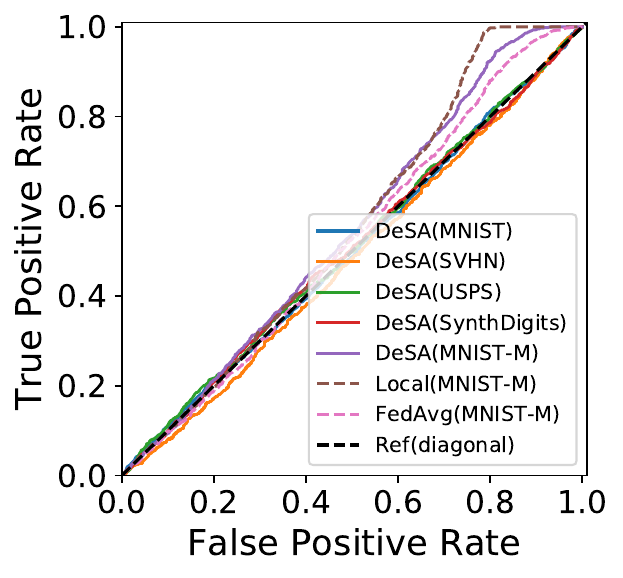}}
\caption{MIA on the models trained by SVHN, SynthDigits, and MNIST-M clients. Observe that the synthetic data sharing of \ours{} does not reveal other clients' local data identity information.}
\label{fig:mia}
\end{figure}

The goal of the experiment is to investigate whether our local model is vulnerable to MIA, namely leaking information about local real datasets' membership. To compare and demonstrate the effectiveness of the chosen attack, we also present results from local training and FedAvg training. We conduct MIA experiments using \texttt{DIGITS}. The MIA for local training and FedAvg is related to real local training data. Since we use synthetic anchor data generated from other clients with data distillation, we also provide MIA results for inferring real data of other clients. For example, if attacking SVHN's local model, local training and FeAvg report the MIA results on SVHN only, while we also report MIA results on MNIST, USPS, SynthDigits, MNIST-M for \ours{}. 

Using the metric in \cite{carlini2022membership}, the results are shown in Figure~\ref{fig:mia}. The Ref(diagonal) line indicates MIA \textbf{cannot} tell the differences between training and testing data. If the line bends towards True Positive Rate, it means the membership form the training set can be inferred. 
It is shown that all the MIA curves of targeted and other cients lie along the Ref line for \ours{}'s model, which indicates that the membership of each training sets is well protected given the applied attack. While the curves for the MIA attacks on FedAvg and local training with SVHN dataset are all offset the Ref (diagonal) line towards True Positive, indicating they are more vulnerable to MIA and leaking training data information.

\subsection{Differential Privacy for Data Synthesis}
\label{app:DP}

To enhance the data privacy-preservation on shared synthetic anchor data, we apply the Differential Privacy stochastic gradient descent (DP-SGD)~\cite{abadi2016deep} for the synthetic image generation.
DP-SGD protects local data information via noise injection on clipped gradients. In our experiments, we apply Gaussian Mechanism for the inejcted noise. 
Specifically, we first sample a class-balanced subset from the raw data to train the objective~\ref{eq:MMD}.
We set up the batch size as 256.
For each iteration, we clip the gradient so that its $l_2$-norm is 2.
The injected noises are from $\mathcal{N}(0, 1.2)$. 
This step ensures $(\epsilon, \delta)$-DP with $(\epsilon, \delta)$ values in \{(7.622, 0.00015), (10.3605, 0.00021), (8.6677, 0.00017), (7.3174, 0.00014), (7.6221, 0.00015)\} guarantees for \{MNIST, SVHN, USPS, SynthDigits, MNIST-M\}, respectively. We visualize the non-DP and DP synthetic images from each clients in \texttt{DIGITS} in Figure~\ref{fig:synthesis_digits} and Figure~\ref{fig:synthesis_digits_dp}, respectively. One can observe that the synthetic data with DP mechanism is noisy and hard to inspect the individual information of the raw data.

\begin{figure}[ht]
	\centering
	\subfloat[Mnist]{  
	   \includegraphics[width=0.19\linewidth]{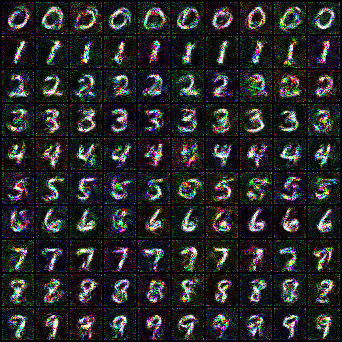}}
	\subfloat[SVHN]{
	   \includegraphics[width=0.19\linewidth]{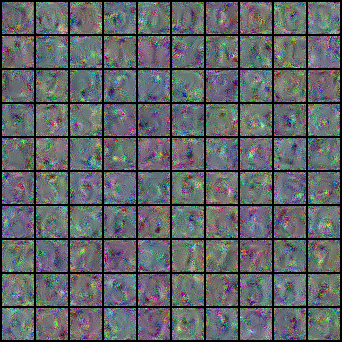}}
	\subfloat[USPS]{
	   \includegraphics[width=0.19\linewidth]{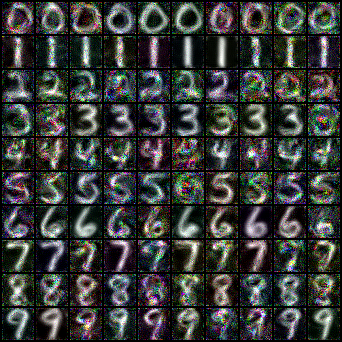}}
	\subfloat[SynthDigits]{
		\includegraphics[width=0.19\linewidth]{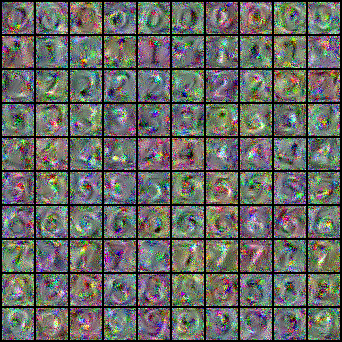}}
        \subfloat[Mnist-M]{
		\includegraphics[width=0.19\linewidth]{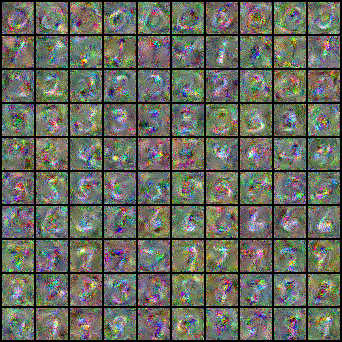}}
	\caption{Visualization of the global and local synthetic images from the \texttt{DIGITS} dataset. (a) visualized the MNIST client; (b) visualized the SVHN client; (c) visualized the USPS client; (d) visualized the SynthDigits client; (e) visualized the MNIST-M client; (f) visualized the server synthetic data.}
		\label{fig:synthesis_digits}
\end{figure}

\begin{figure}[ht]
	\centering
	\subfloat[Mnist]{  
	   \includegraphics[width=0.19\linewidth]{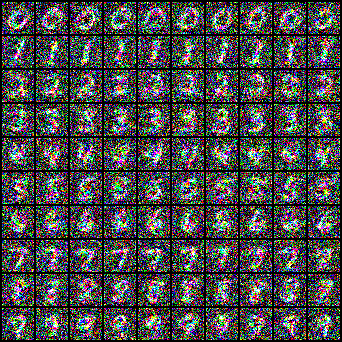}}
	\subfloat[SVHN]{
	   \includegraphics[width=0.19\linewidth]{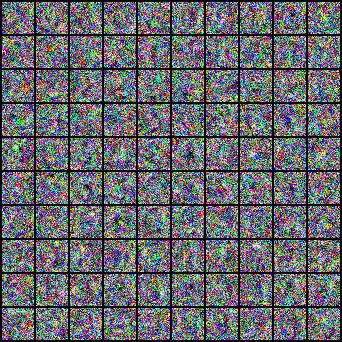}}
	\subfloat[USPS]{
	   \includegraphics[width=0.19\linewidth]{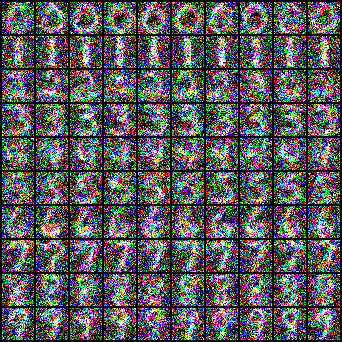}}
	\subfloat[SynthDigits]{
		\includegraphics[width=0.19\linewidth]{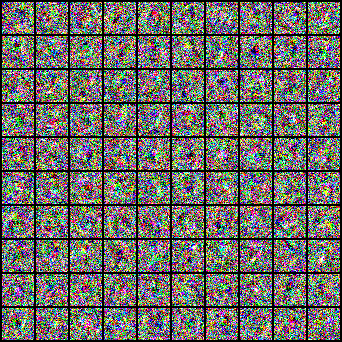}}
        \subfloat[Mnist-M]{
		\includegraphics[width=0.19\linewidth]{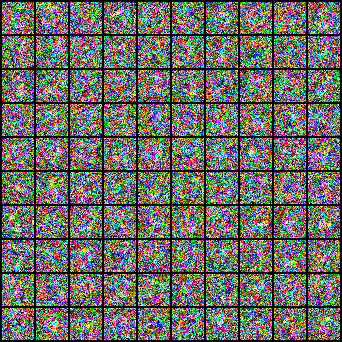}}
	\caption{Visualization of the global and local synthetic images from the \texttt{DIGITS} dataset with \textbf{DP mechanism}. (a) visualized the MNIST client; (b) visualized the SVHN client; (c) visualized the USPS client; (d) visualized the SynthDigits client; (e) visualized the MNIST-M client; (f) visualized the server synthetic data.}
		\label{fig:synthesis_digits_dp}
\end{figure}

We replace the synthetic data by DP synthetic data and perform \texttt{DIGITS} experiments, and the result is shown in Table~\ref{tab:DP_compare}. It can be observed that although \ours{}'s performance slightly drops due to the DP mechanism, the averaged inter and intra-accuracy are in the second place, which indicates that \ours{} is robust as long as we can synthesize images that roughly captures the global data distribution.

\begin{table}[h]
\centering
\caption{We add the the results for \ours{} trained with DP synthetic anchor data into our Table~\ref{tab:hetero_model}. The best result is marked as \textbf{bold}, and the second best is marked as \textcolor{blue}{blue}. The table shows that \ours{} with DP synthetic anchor data can still obtain comparable results as \ours{} with non-DP synthetic data.}
\resizebox{0.7\linewidth}{!}{
\begin{tabular}{ll|llllll}
\toprule
\multicolumn{2}{l|}{\multirow{2}{*}{}}                 & \multicolumn{6}{l}{DIGITS}                       \\ \cline{3-8} 
\multicolumn{2}{l|}{}                                  & \multicolumn{1}{l|}{MN(C)} & \multicolumn{1}{l|}{SV(A)} & \multicolumn{1}{l|}{US(C)} & \multicolumn{1}{l|}{Syn(A)} & \multicolumn{1}{l|}{MM(C)} & Avg   \\ \hline
\multicolumn{2}{l|}{FedHe}                             & \multicolumn{1}{l|}{59.51} & \multicolumn{1}{l|}{66.67} & \multicolumn{1}{l|}{49.89} & \multicolumn{1}{l|}{75.39}  & \multicolumn{1}{l|}{71.57} & 64.81 \\ \hline
\multicolumn{1}{l|}{\multirow{2}{*}{FedDF}} & Cifar100 & \multicolumn{1}{l|}{65.98}     & \multicolumn{1}{l|}{65.21}     & \multicolumn{1}{l|}{61.30}     & \multicolumn{1}{l|}{69.65}      & \multicolumn{1}{l|}{\textbf{74.48}}     & 67.32     \\ \cline{2-8} 
\multicolumn{1}{l|}{}                       & FMNIST   & \multicolumn{1}{l|}{43.05}     & \multicolumn{1}{l|}{69.14}     & \multicolumn{1}{l|}{44.95}     & \multicolumn{1}{l|}{74.67}      & \multicolumn{1}{l|}{71.27}     & 60.61     \\ \hline
\multicolumn{1}{l|}{\multirow{2}{*}{FCCL}}  & CIFAR100 & \multicolumn{1}{l|}{-}     & \multicolumn{1}{l|}{-}     & \multicolumn{1}{l|}{-}     & \multicolumn{1}{l|}{-}      & \multicolumn{1}{l|}{-}     & -     \\ \cline{2-8} 
\multicolumn{1}{l|}{}                       & FMNIST   & \multicolumn{1}{l|}{46.43}     & \multicolumn{1}{l|}{61.02}     & \multicolumn{1}{l|}{42.64}     & \multicolumn{1}{l|}{63.05}      & \multicolumn{1}{l|}{66.39}     & 55.91     \\ \hline
\multicolumn{2}{l|}{FedProto}                          & \multicolumn{1}{l|}{62.59} & \multicolumn{1}{l|}{\textcolor{blue}{71.74}} & \multicolumn{1}{l|}{58.52} & \multicolumn{1}{l|}{\textbf{81.19}}  & \multicolumn{1}{l|}{\textcolor{blue}{74.44}} & 69.70 \\ \hline
\multicolumn{2}{l|}{\ours($D_{\rm VHL}^{Syn}$)}                              & \multicolumn{1}{l|}{54.40} & \multicolumn{1}{l|}{62.03} & \multicolumn{1}{l|}{42.34} & \multicolumn{1}{l|}{67.75}  & \multicolumn{1}{l|}{73.03} & 59.91 \\ \hline
\multicolumn{2}{l|}{\ours}                              & \multicolumn{1}{l|}{\textbf{70.12}} & \multicolumn{1}{l|}{\textbf{76.17}} & \multicolumn{1}{l|}{\textbf{71.17}} & \multicolumn{1}{l|}{\textcolor{blue}{81.10}}  & \multicolumn{1}{l|}{73.83} & \textbf{74.47} \\ \hline
\multicolumn{2}{l|}{\ours(DP)}                          & \multicolumn{1}{l|}{\textcolor{blue}{69.06}} & \multicolumn{1}{l|}{71.54} & \multicolumn{1}{l|}{\textcolor{blue}{63.92}} & \multicolumn{1}{l|}{78.93}  & \multicolumn{1}{l|}{73.16} & \textcolor{blue}{71.12} \\ \bottomrule
\end{tabular}
}
\label{tab:DP_compare}
\end{table}

\section{Local Epoch}
\label{app:localepoch}
Here we present the effect of local epochs on \ours{}. To ensure fair comparison, we fix the total training iterations for the three experiments, \ie we set FL communication rounds to 50 when local epochs is 2 to match up with local epoch equals to 1 when FL communication rounds is 100. Figure~\ref{tab:epoch} shows that \ours{} is robust to various local epoch selections. The experiment is run on \texttt{DIGITS} dataset, and we report the global accuracy.

\begin{table}[]
\centering
\caption{Ablation study on number of local epochs. The experiment is run on \texttt{DIGITS} dataset.}
\resizebox{0.4\linewidth}{!}{
\begin{tabular}{l|l|l|l}
\hline
Local Epoch & 1     & 2     & 5     \\ \hline
Global Acc  & 74.47 & 74.15 & 74.34 \\ \hline
\end{tabular}
}
\label{tab:epoch}
\end{table}

\section{Communication Overhead}
\label{app:communication}

As noted in Section~\ref{sec:setting}, \ours{} only requires sharing logits w.r.t. Global synthetic data during training. Thus it has a relatively low communication overhead compared to baseline methods which require sharing model parameters. For fair comparison, we analyze the communication cost based on the number of parameters Pre-FL and During-FL in Table~\ref{tab:communication}. Note that we show the number of parameters for one communication round for During-FL, and the total communication cost depends on the number of global iterations. One can observe that sharing logits can largely reduce the communication overhead.  For example, if we use ConvNet as our model, set IPC=50, and train for 100 global iteration, the total number of parameters for communication for DeSA will be 30.7 K $\times$ 50 (Pre-FL) + 10 (number of classes) $\times$ 50 (images/class) $\times$ 10 (logits/image) $\times$ 100 (global iteration) = 2.04M. In comparison, baseline methods need to share 0 (Pre-FL) + 320K (parameters/iteration) $\times$ 100 (global iteration) = 32M, which is much larger than DeSA. Under model heterogeneity experimental setting, clients using AlexNet would suffer even higher total communication cost, which is 0 (Pre-FL) + 1.87M (parameters/iteration) $\times$ 100 (global iteration) = 187M.

\begin{table}[ht]
\centering
\caption{Comparison of communication overhead. Note that for \ours{}, we only share virtual global anchor logits during training. The total communication cost counts the total parameter transferred for 100 global iterations. IPC is the synthesized images per class, and C is the number of classes.}
\begin{tabular}{l|l|l|l}
\toprule
    & ConvNet & AlexNet & Global Anchor Logits \\ \hline
Pre-FL & 0   & 0  & 30.7 K $\times$ IPC $\times$ C    \\ \hline
During-FL & 320 K   & 1.87 M  & 100 $\times$ IPC  $\times$ C    \\ \hline
Total & 32M   & 187M  & 40.7K $\times$ IPC $\times$ C     \\ \bottomrule
\end{tabular}
\label{tab:communication}
\end{table}

\section{Datasets and Synthetic Images}
\label{app:synthetic_images}

\begin{figure}[h]
	\centering
	\subfloat[]{  
		\includegraphics[width=0.198\linewidth]{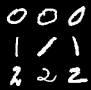}}
	\subfloat[]{
	\includegraphics[width=0.195\linewidth]{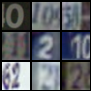}}
	\subfloat[]{
	\includegraphics[width=0.2\linewidth]{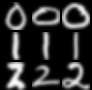}}
	\subfloat[]{
	\includegraphics[width=0.195\linewidth]{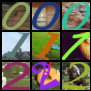}}
        \subfloat[]{
	\includegraphics[width=0.195\linewidth]{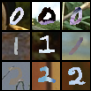}}
  
	\caption{Visualization of the original digits dataset. (a) visualized the MNIST client; (b) visualized the SVHN client; (c) visualized the USPS client; (d) visualized the SynthDigits client; (e) visualized the MNIST-M client.}
		\label{fig:original_digits}
\end{figure}

\begin{figure}[ht]
	\centering
	\subfloat[Amazon]{  
	   \includegraphics[width=0.24\linewidth]{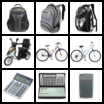}}
	\subfloat[Caltech]{
	   \includegraphics[width=0.24\linewidth]{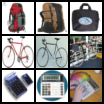}}
	\subfloat[DSLR]{
	   \includegraphics[width=0.24\linewidth]{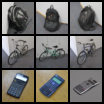}}
	\subfloat[Webcam]{
		\includegraphics[width=0.24\linewidth]{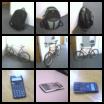}}
	\caption{Visualization of the original digits dataset. (a) visualized the Amazon client; (b) visualized the Caltech client; (c) visualized the DSLR client; (d) visualized the Webcam client}
		\label{fig:original_office}
\end{figure}

\begin{figure}[h]
	\centering
 \subfloat[]{  
		\includegraphics[width=0.155\linewidth]{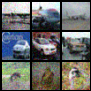}}
	\subfloat[]{  
		\includegraphics[width=0.157\linewidth]{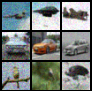}}
	\subfloat[]{
		\includegraphics[width=0.155\linewidth]{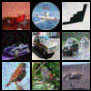}}
	\subfloat[]{
	\includegraphics[width=0.157\linewidth]{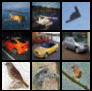}}
        \subfloat[]{
	\includegraphics[width=0.155\linewidth]{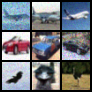}}
        \subfloat[]{
	\includegraphics[width=0.157\linewidth]{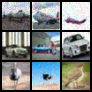}}
	\caption{Visualization of the original \texttt{CIFAR10C}. Sampled images from the first six clients.}
    \label{fig:original_cifar10c}
\end{figure}

\paragraph{Detailed Information of Selected Datasets}
1) \texttt{DIGITS}=\{MNIST~\cite{lecun1998gradient}, SVHN~\cite{netzer2011reading}, USPS~\cite{hull1994database}, SynthDigits~\cite{ganin2015unsupervised}, MNIST-M~\cite{ganin2015unsupervised}\} consists of 5 digit datasets with handwritten, real street and synthetic digit images of $0, 1, \cdots, 9$. Thus, we assume 5 clients for this set of experiments. We use \texttt{DIGITS} as baseline to show \ours{} can handle FL under large domain shift. Example images can be found in Figure~\ref{fig:original_digits}.\\
2) \texttt{OFFICE}=\{Amazon~\cite{saenko2010adapting}, Caltech~\cite{griffin2007caltech}, DSLR~\cite{saenko2010adapting}, and WebCam~\cite{saenko2010adapting}\} consists of four data sources from Office-31~\cite{saenko2010adapting} (Amazon, DSLR, and WebCam) and Caltech-256~\cite{griffin2007caltech} (Caltech), resulting in four clients. Each client possesses images that were taken using various camera devices in different real-world environments, each featuring diverse backgrounds. We show \ours{} can handle FL under large domain shifted \texttt{real-world} images using \texttt{OFFICE}. Example images can be found in Figure~\ref{fig:original_office}.\\ 
3) \texttt{CIFAR10C} consists subsets extracted from Cifar10-C~\cite{hendrycks2019benchmarking}, a collection of augmented Cifar10~\cite{krizhevsky2009learning} that applies 19 different corruptions. We employ a Dirichlet distribution with $\beta=2$ for the purpose of generating three partitions within each distorted non-IID dataset. As a result, we have 57 clients with domain- and label-shifted datasets. Example images can be found in Figure~\ref{fig:original_cifar10c}.

\paragraph{Synthetic Data Generation}
We fix ConvNet as the backbone for data synthesis to avoid additional domain shift caused by different model architectures. We set learning rate to 1 and use SGD optimizer with momentum = 0.5. The batch size for \texttt{DIGITS} and \texttt{CIFAR10} is set to 256, while we use 32 for \texttt{OFFICE} as it's clients has fewer data points. For the same reason, we use 500 synthetic data points for \texttt{DIGITS} and \texttt{CIFAR10C}, and we set 100 synthetic data points for \texttt{OFFICE}. The training iteration for \texttt{DIGITS} and \texttt{OFFICE} is 1000, and we set 2000 for \texttt{CIFAR10C} since it contains more complex images.

We show the local synthetic images and global anchor images of \texttt{DIGITS}, \texttt{OFFICE}, and \texttt{CIFAR10C} in Figure~\ref{fig:sampled_synthesis_digits}, Figure~\ref{fig:sampled_synthesis_office}, and Figure~\ref{fig:sampled_synthesis_cifar10c}, respectively.

\begin{figure}[h]
	\centering
	\subfloat[Mnist]{  
		\includegraphics[width=0.155\linewidth]{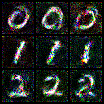}}
	\subfloat[SVHN]{
	\includegraphics[width=0.155\linewidth]{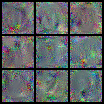}}
	\subfloat[USPS]{
		\includegraphics[width=0.155\linewidth]{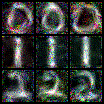}}
	\subfloat[SynthDigits]{
		\includegraphics[width=0.155\linewidth]{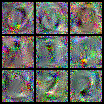}}
        \subfloat[Mnist-M]{
		\includegraphics[width=0.155\linewidth]{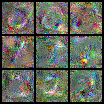}}
        \subfloat[Average]{
		\includegraphics[width=0.155\linewidth]{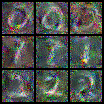}}
	\caption{Visualization of the sampled global and local synthetic images from the \texttt{DIGITS} dataset. (a) visualized the MNIST client's synthetic data; (b) visualized the SVHN client's synthetic data; (c) visualized the USPS client's synthetic data; (d) visualized the SynthDigits client's synthetic data; (e) visualized the MNIST-M client's synthetic data; (f) visualized the server synthetic data.}
		\label{fig:sampled_synthesis_digits}
\end{figure}

\begin{figure}[h]
	\centering
	\subfloat[Amazon]{  
		\includegraphics[width=0.18\linewidth]{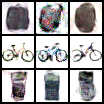}}
	\subfloat[Caltech]{
	\includegraphics[width=0.18\linewidth]{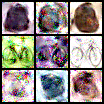}}
	\subfloat[DSLR]{
		\includegraphics[width=0.18\linewidth]{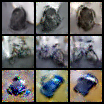}}
	\subfloat[Webcam]{
		\includegraphics[width=0.18\linewidth]{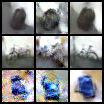}}
        \subfloat[Average]{
		\includegraphics[width=0.18\linewidth]{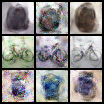}}
	\caption{Visualization of the sampled global and local synthetic images from the \texttt{OFFICE} dataset. (a) visualized the Amazon client's synthetic data; (b) visualized the Caltech client's synthetic data; (c) visualized the DSLR client's synthetic data; (d) visualized the Webcam client's synthetic data; (e) visualized the averaged synthetic data.}
		\label{fig:sampled_synthesis_office}
\end{figure}

\begin{figure}[h]
	\centering
	\subfloat[Client0]{  
		\includegraphics[width=0.155\linewidth]{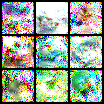}}
	\subfloat[Client1]{
	\includegraphics[width=0.155\linewidth]{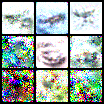}}
	\subfloat[Client2]{
		\includegraphics[width=0.155\linewidth]{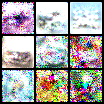}}
	\subfloat[Client3]{
		\includegraphics[width=0.155\linewidth]{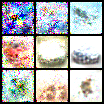}}
        \subfloat[Client4]{
		\includegraphics[width=0.155\linewidth]{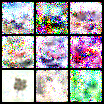}}
        \subfloat[Average]{
		\includegraphics[width=0.155\linewidth]{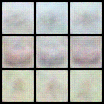}}
	\caption{Visualization of the sampled global and local synthetic images from the first 5 clients in \texttt{CIFAR10C} dataset. (a) visualized the first 's synthetic data; (b) visualized the second client's synthetic data; (c) visualized the third client's synthetic data; (d) visualized the forth client's synthetic data; (e) visualized the fifth client's synthetic data; (f) visualized the server synthetic data.}
		\label{fig:sampled_synthesis_cifar10c}
\end{figure}

\newpage
\section{Model Architectures}
\label{app:model_arch}

We use ConvNet to perform data distillation for the best synthesis quality. For model hetero genenity scenarios, we randomly select classification model architectures from \{AlexNet, ConvNet\}. The detailed setup for bot models are depicted in Table~\ref{tab:alexnet} and Table~\ref{tab:convnet}

\begin{table}[ht]
\caption{AlexNet architecture. For the convolutional layer (Conv2D), we list parameters with a sequence of input and output dimensions, kernel size, stride, and padding. For the max pooling layer (MaxPool2D), we list kernel and stride. For a fully connected layer (FC), we list input and output dimensions.}
\label{tab:alexnet}
\renewcommand\arraystretch{1.2}
\centering
\scalebox{0.99}{
\begin{tabular}{l|cccc}
\toprule
\multicolumn{1}{c|}{\multirow{1}{*}{Layer}} & \multicolumn{4}{c}{Details} \\ \hline
\multicolumn{1}{c|}{\multirow{1}{*}{1}}& \multicolumn{4}{c}{Conv2D(3, 128, 5, 1, 4), ReLU, MaxPoo2D(2, 2)} \\ \hline
\multicolumn{1}{c|}{\multirow{1}{*}{2}}& \multicolumn{4}{c}{Conv2D(128, 192, 5, 1, 2), ReLU, MaxPoo2D(2, 2)} \\ \hline
\multicolumn{1}{c|}{\multirow{1}{*}{3}}& \multicolumn{4}{c}{Conv2D(192, 256, 3, 1, 1), ReLU} \\ \hline

\multicolumn{1}{c|}{\multirow{1}{*}{4}}& \multicolumn{4}{c}{Conv2D(256, 192, 3, 1, 1), ReLU} \\ \hline
\multicolumn{1}{c|}{\multirow{1}{*}{5}}& \multicolumn{4}{c}{Conv2D(192, 192, 3, 1, 1), ReLU, MaxPoo2D(2, 2)} \\ \hline

\multicolumn{1}{c|}{\multirow{1}{*}{22}}& \multicolumn{4}{c}{FC(3072, num\_class)} \\ \bottomrule
\end{tabular}%
}
\end{table}

\begin{table}[ht]
\caption{ConvNet architecture. For the convolutional layer (Conv2D), we list parameters with a sequence of input and output dimensions, kernel size, stride, and padding. For the max pooling layer (MaxPool2D), we list kernel and stride. For a fully connected layer (FC), we list the input and output dimensions. For the GroupNormalization layer (GN), we list the channel dimension.}
\label{tab:convnet}
\renewcommand\arraystretch{1.2}
\centering
\scalebox{0.99}{
\begin{tabular}{l|cccc}
\toprule
\multicolumn{1}{c|}{\multirow{1}{*}{Layer}} & \multicolumn{4}{c}{Details} \\ \hline
\multicolumn{1}{c|}{\multirow{1}{*}{1}}& \multicolumn{4}{c}{Conv2D(3, 128, 3, 1, 1), GN(128), ReLU, AvgPool2d(2,2,0)} \\ \hline

\multicolumn{1}{c|}{\multirow{1}{*}{2}}& \multicolumn{4}{c}{Conv2D(128, 118, 3, 1, 1), GN(128), ReLU, AvgPool2d(2,2,0)} \\ \hline

\multicolumn{1}{c|}{\multirow{1}{*}{3}}& \multicolumn{4}{c}{Conv2D(128, 128, 3, 1, 1), GN(128), ReLU, AvgPool2d(2,2,0)} \\ \hline

\multicolumn{1}{c|}{\multirow{1}{*}{4}}& \multicolumn{4}{c}{FC(1152, num\_class)} \\ \bottomrule
\end{tabular}%
}
\end{table}

\section{More Related Work}
\label{app:related}

\subsection{Model Homogeneous Federated Learning}
We list down different Model Homogeneous FL approaches in decentralized FL and collaborative methods that are relevant to our setting.
\subsubsection{Decentralized Federated Learning}
In order to tackle training a global model without a server, Decentralized FL methods communicate a set of models through diverse decentralized client-network topologies (such as a ring - \cite{chang2018distributed}, Mesh - \cite{roy2019braintorrent}, or a sequential line \cite{assran2019stochastic}) using different communication protocols such as Single-peer(gossip) or Multiple-Peer(Broadcast). ~\cite{yuan2023peer, sheller2019multi,sheller2020federated} pass a single model from client to client similar to an Incremental Learning setup. In this continual setting, only a \textbf{single model} is trained. \cite{pappas2021ipls,roy2019braintorrent,assran2019stochastic} pass models and aggregate their weights similar to conventional FL.  Since these models use averaged aggregation techniques similar to FedAvg, most of these methods assume  client \textbf{model homogeneity}. \ours{}'s client network topology is similar to that of a Mesh using the broadcast-gossip protocol, where every client samples certain neighbours in each communication round for sharing logits.

None of the works above aim to train various client model types without a server, which is our goal.

\subsubsection{Collaborative Methods}
\cite{fallah2020personalized} uses an MAML(model agnostic meta learning) framework to explicitly train model homogeneous client models to personalize well. The objective function of MAML evaluates the personalized performance assuming a one-step gradient descent update on the subsequent task. \cite{huang2021personalized} modifies the personalized objective by adding an \text{attention inducing term} to the objective function which promotes collaboration between pairs of clients that have similar data.

\cite{ghosh2022efficient} captures settings where different groups of users have their own objectives (learning tasks) but by aggregating their
private data with others in the same cluster (same learning task), they
can leverage the strength in numbers in order to perform
more efficient personalized federated learning 
\cite{donahue2021model} uses game theory to analyze whether a client should jointly train with other clients in a conventional FL setup [\ref{subsec: Conventional FL}] assuming it's primary objective is to minimize the MSE loss on its own private dataset. They also find techniques where it is more beneficial for the clients to create coalitions and train one global model.

All the above works either slightly change the intra-client objective to enable some collaboration between model-homogeneous clients or explicitly create client clusters to collaboratively learn from 
each other. They do not tackle the general objective function that we do- \ref{eq:fmd}

\subsection{Model Heterogeneous Federated Learning}
Model heterogeneous FL approaches relevant to \ours{} broadly come under the following two types.
\subsubsection{Knowledge distillation methods}
\cite{gong2022preserving} proposes FedKD that is a one-shot centralized Knowledge distillation approach on unlabelled public data after the local training stage in-order to mitigate the accuracy drop due to the label shift amongst clients.
DENSE~\cite{zhang2022dense} propose one-shot federated learning to generate decision boundary-aware synthetic data and train the global model on the server side.
FedFTG~\cite{zhang2022fine} finetunes the global model by knowledge distillation with hard sample mining. \cite{yang2021regularized} introduces a method called Personalized Federated Mutual Learning (PFML), which leverages the non-IID properties to create customized models for individual parties. PFML incorporates mutual learning into the local update process within each party, enhancing both the global model and personalized local models. Furthermore, mutual distillation is employed to expedite convergence. The method assumes homogeneity of models for global server aggregation. However, all the above methods are centralized.

\subsubsection{Mutual Learning Methods}

Papers in this area predominantly use ideas from deep-mutual learning \cite{zhang2018deep}
\cite{matsuda2022fedme} uses deep mutual learning to train heterogeneous local models for the sole purpose of personalization. The method creates clusters of clients whose local models have similar outputs. Clients within a cluster exchange their local models in-order to tackle label shift amongst the data points. However, the method is centralized and each client maintains two copies of models, one which is personalized and one that is exchanged. \cite{li2021decentralized} has a similar setting to \cite{chan2021fedhe}, but instead solves the problem in a peer to peer decentralized manner using soft logit predictions on the local data of a client itself.  It makes its own baselines that assume model homogeneity amongst clients, also their technique assumes that there is no covariate shift because it only uses local data for the soft predictions. However, their technique can be modified for model heterogeneity. They report personalization(Intra) accuracies only.

\subsection{Dataset Distillation}

Data distillation methods aim to create concise data summaries $D_{syn}$ that can effectively substitute the original dataset $D$ in tasks such as model training, inference, and architecture search.
Moreover, recent studies have justified that data distillation also preserves privacy \cite{dong2022privacy, carlini2022no} which is
critical in federated learning. In practice, dataset distillation is used in healthcare for medical data
sharing for privacy protection \cite{li2022dataset}. We briefly mention two types of Distillation works below.

\subsubsection{Gradient and Trajectory Matching techniques}
Gradient Matching~\cite{zhao2020dataset} is proposed to make the deep neural network produce similar gradients for both the terse synthetic images and the original large-scale dataset. The objective function involves matching the gradients of the loss w.r.t weights(parameters) evaluated on both $D$ and $D_{syn}$ at successive parameter values during the optimization on the original dataset $D$. Usually the cosine distance is used to measure the difference in gradient direction.
Other works in this area modify the objective function slightly, by either adding class contrastive signals for better stability \cite{lee2022dataset} or by adding same image-augmentations(such as crop, rotate  to both $D$ and $D_{syn}$)\cite{zhao2021dataset}.
A similar technique is that of \cite{cazenavette2022dataset} which tries to match the intermediate parameters in the optimization trajectory of both $D$ and $D_{Syn}$. It is very \textit{computationally expensive} because of a gradient unrolling in the optimization. TESLA \cite{cui2023scaling} attempts at using linear-algebraic manipulations to give better computational guarantees for Trajectory matching

\subsubsection{Distribution Matching techniques}
Distribution matching \cite{zhao2023dataset}
solves the distillation task via a single-level optimization, leading to a \textit{vastly improved scalability}. More
specifically, instead of matching the quality of models on $D$ vs. $D_{syn}$, distribution-matching techniques directly match the distribution of $D$ vs. $D_{syn}$ in a latent encoded space. See \ref{eq:MMD} for the objective function.
CAFE ~\cite{wang2022cafe} further refines the distribution-matching idea by solving a bilevel optimization problem for jointly optimizing a single encoder and the data summary,
rather than using a pre-determined set of encoders
Adversarial techniques using Distribution matching such as IT-GAN \cite{zhao2022synthesizing} and GAN \cite{goodfellow2014generative} aren't suitable for a serverless setting.
Since we aim to mitigate drifts in client-distribution across using our synthetic data, 
Distribution Matching is a more natural option for our work.


\newpage

\end{document}